\documentclass[letterpaper]{article} 
\usepackage[submission]{aaai23}  
\usepackage{times}  
\usepackage{helvet}  
\usepackage{courier}  
\usepackage[hyphens]{url}  
\usepackage{graphicx} 
\usepackage{natbib}  
\usepackage{caption} 
\usepackage{algorithm,algorithmic,cuted}
\usepackage{multirow,kantlipsum}
\usepackage{graphics,subcaption,color}
\usepackage{amssymb,bm,amsmath,amsthm,amsfonts,mathrsfs}

\newtheorem{thm}{Theorem}
\newtheorem{lemma}[thm]{Lemma}
\newtheorem{definition}{Definition}

\newtheorem{cor}{Corollary}

\newcommand{\mbR}{\mathbb{R}}

\newcommand{\tr}{\rm{Tr}}

\newcommand{\mcB}{\mathcal{B}}

\newcommand{\mcN}{\mathcal{N}}

\newcommand{\mcS}{\mathcal{S}}

\newcommand{\card}{{\rm card}}

\newcommand{\mbE}{\mathbb{E}}

\usepackage{newfloat}
\usepackage{listings}
\floatstyle{ruled}
\newfloat{listing}{tb}{lst}{}
\floatname{listing}{Listing}
\title{Mutual Information Learned Classifiers: an Information-theoretic Viewpoint of Training Deep Learning Classification Systems}
\author{T
    Jirong Yi\textsuperscript{\rm 1}\textsuperscript{\rm 2}\thanks{Corresponding emails should be sent to: jirong-yi@uiowa.edu, jirong.yi@hologic.com},
    Qiaosheng Zhang\textsuperscript{\rm 3},
    Zhen Chen\textsuperscript{\rm 4},
    Qiao Liu\textsuperscript{\rm 5},
    Wei Shao\textsuperscript{\rm 5}
}
\affiliations{T
    \textsuperscript{\rm 1} University of Iowa
    \textsuperscript{\rm 2} Hologic Inc
    \textsuperscript{\rm 3} National University of Singapore
    \textsuperscript{\rm 4} University of California at Irvine
    \textsuperscript{\rm 5} Stanford University
}

\usepackage{bibentry}

\begin{document}

\maketitle

\begin{abstract}
	Deep learning systems have been reported to achieve state-of-the-art performances in many applications, and a key is the existence of well trained classifiers on benchmark datasets. As a main-stream loss function, the cross entropy can easily lead us to find models which demonstrate severe overfitting behavior. In this paper, we show that the existing cross entropy loss minimization problem essentially learns the label conditional entropy (CE) of the underlying data distribution of the dataset. However, the CE learned in this way does not characterize well the information shared by the label and the input. In this paper, we propose a mutual information learning framework where we train deep neural network classifiers via learning the mutual information between the label and the input. Theoretically, we give the population classification error lower bound in terms of the mutual information. In addition, we derive the mutual information lower and upper bounds for a concrete binary classification data model in $\mathbb{R}^n$, and also the error probability lower bound in this scenario. Empirically, we conduct extensive experiments on several benchmark datasets to support our theory. The mutual information learned classifiers (MILCs) achieve far better generalization performances than the conditional entropy learned classifiers (CELCs) with an improvement which can exceed more than 10\% in testing accuracy.
\end{abstract}

\section{Introduction}\label{Sec:Introduction}

Ever since the breakthrough made by Krizhevsky et al. \cite{krizhevsky_imagenet_2012}, deep learning has been finding tremendous applications in different areas such as computer vision and traditonal signal processing \cite{goodfellow_deep_2016,bora_compressed_2017,yi_outlier_2018,zhou_global_2022,zhou_rethinking_2022,zheng_progressive_2022,yan_lawin_2022,xu_closing_2022,khan_deep_2020}. In nearly all of these applications, a fundamental classification task is usually involved. For many other applications which do not directly or explicitly involve classifications, they still use models that are pretrained via classification tasks as a backbone for extracting useful and meaningful representation for specific tasks \cite{liu_swin_2021,he_deep_2015}. In practice, such extracted representations have been reported to be beneficial for the downstream tasks \cite{liu_swin_2021}.

To train classification models, the machine learning community has been mainly using cross entropy loss or its variants (with some regularization term). However, the models that are trained using cross entropy can easily overfit the data and result in pretty bad generalization performance. This motivates the proposal of many techniques for improved generalization. One important line of works comes from the regularization viewpoint, i.e., restricting the model space for searching during training to avoid overfitting to noise \cite{goodfellow_deep_2016}. Examples include weight decay (or $\ell_2$ regularization), and $\ell_1$ regularization. A  major limitation of the regularization approaches is that they are essentially incorporating prior knowledge about the learning tasks to be solved, but unfortunately, such prior knowledge is not always available in practice, or may not be optimal even if it exists. What makes things worse is that as recently reported, such prior knowledge may make the learned models adversarially vulnerable so that adversarial attacks can be easily achieved. This is because the model can be underfitted to those unseen adversarial examples \cite{xu_adversarial_2019}.

In this paper, we show that the existing cross entropy loss minimization for training deep neural network classifiers essentially learns the conditional entropy of the underlying data distribution of the input and the label. We argue that this can be the fundamental reason which accounts for severe overfitting of models trained by cross entropy loss minimization, and the extremely small training loss in practice implies that the learned model completely ignores the information remained the label after revealing the input to it. However, this is not always the case.  To see these, we consider the MNIST classification task \cite{goodfellow_deep_2016}. When the image of a hand-written digital is given, we are 100\% sure for most of the time which class the image belongs to. In Figure \ref{Fig:MNISTImgsSure1}, we show one of such images on which the digit is no doubt 1, thus the remaining uncertainty or information about label when this image is given is 0. However, this is not always true because we can have image samples whose classes cannot determined with complete certainty. The digit in Figure \ref{Fig:MNISTImgsAmbug1} has a truth label 1, but it looks like 2. Different people can have different labels for these image samples, but their ground truth labels are at the discretion of the creator of them. In more complex image classification tasks such as ImageNet classification (See Appendix for illustrative image samples), a single image itself can contain multiple objects, and thus belongs to multiple classes. However, it has only single annotation or label which depends on the discretion of the human annotators \cite{deng_imagenet:_2009}, which also leads to the ignorance of label conditional entropy. 

\begin{figure}[!htb]
	\centering
	\begin{subfigure}[b]{0.1\textwidth}
		\centering
		\includegraphics[width=\linewidth]{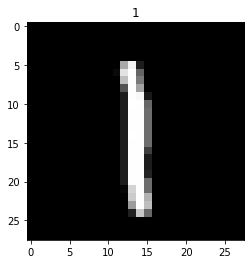}
		\caption{Truth label is 1}\label{Fig:MNISTImgsSure1}
	\end{subfigure}
	~
	\begin{subfigure}[b]{0.1\textwidth}
		\centering
		\includegraphics[width=\linewidth]{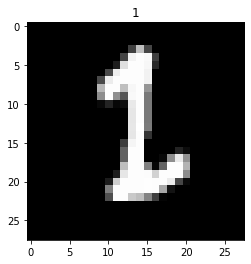}
		\caption{Truth label is 1}\label{Fig:MNISTImgsAmbug1}
	\end{subfigure}
	~
	\begin{subfigure}[b]{0.1\textwidth}
		\centering
		\includegraphics[width=\linewidth]{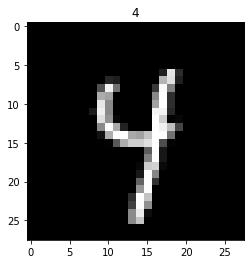}
		\caption{Truth label is 4}\label{Fig:MNISTImgsAmbug4}
	\end{subfigure}%
	~
	\begin{subfigure}[b]{0.1\textwidth}
		\centering
		\includegraphics[width=\linewidth]{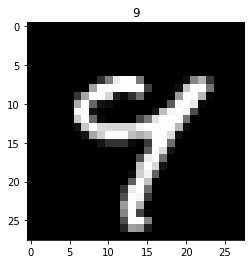}
		\caption{Truth label is 9}\label{Fig:MNISTImgsAmbug9}
	\end{subfigure}
	\caption{Image examples from MNIST dataset.}\label{Fig:MNISTImgs}
\end{figure}

We propose a new learning framework where classifiers are trained via learning the mutual information of the dataset. Under our framework, we design a new loss for training the deep neural network where the loss itself originates from a representation of mutual information of the dataset generating distribution, and we propose new pipelines associated with the proposed framework for efficient learning and inference. We refer to the corresponding loss as mutual information learning loss (milLoss), and the traditional cross entropy loss as conditional entropy learning loss (celLoss) since it essentially learns the conditional entropy of the dataset generating distribution (we will show this in later sections). When reformulated as a regularized form of the celLoss, the milLoss encourages the model not only to accurately learn the conditional entropy of the label when an input is given, but also to precisely learn the entropy of the label. This is distinctly different from the label smoothing regularization (LSR), confidence penalty (CP), label correction (LC) etc which consider the conditional entropy of the label \cite{szegedy_going_2015,wang_proselflc_2021,meister_generalized_2020}.


For the proposed mutual information learning paradigm, we establish a population error probability lower bound for classification, in terms of the mutual information (MI) between the input and the label by using Fano's inequality \cite{cover_elements_2012}. This explicitly characterizes how the performance of the models is connected to the mutual information contained in the dataset. Compared to Fano's inequality,  our bound is tighter due to a carefully designed relaxation (See Appendix for the details). Our error probability bound is applicable for arbitrary distribution and arbitrary learning algorithms. We also consider a concrete binary classification problem in $\mbR^n$, and derive both lower and upper bounds on the mutual information of the data distribution. We also derive error probability bounds for this binary classification data model. We conduct extensive experiments to validate our theoretical analysis by using classification tasks on benchmark dataset such as MNIST and CIFAR-10. The empirical results show that the proposed framework achieves superior generalization performance than the existing conditional entropy learning approach. 

The contributions of this work are summarized as follows.
\begin{itemize}
	\item We show that the existing cross entropy loss minimization for training deep neural networks essentially learn the conditional entropy $H(Y|X)$, and we point out some fundamental limitations of this approach which motivate us to propose a new training paradigm via learning mutual information.
	
	\item For the proposed framework, we establish theoretically the connection between the error probability over the distribution and its mutual information. To better appreciate the proposed framework, we consider a concrete binary classification data model, and derive both lower and uppper bounds for the mutual information and an error probability bound associated with the data distribution.
	
	\item We conduct extensive experiments with training deep neural classifiers on several benchmark datasets to validate our theory. Empirical results show that the proposed framework can improve greatly the generalization performance of deep neural network classifiers.
	
\end{itemize}

\subsection{Related Works}

Our work is highly related to the following several works, but there are distinct differentiations between our work and them. First of all, in 2019, Yi et al. formulated the classification problem under the encoding-decoding paradigm as in Figure \ref{Fig:ErrorProbability} by assuming there is an observation synthesis process which can generate observations or inputs for a given label, and the classification task is simply about inferring the label from the observation. Under this framework, they give theoretical characterizations of the robustness of machine learning models \cite{yi_trust_2019}. They also characterized the limit of an arbitrary adversarial attacking algorithm for attacking arbitrary machine learning systems, and tried to answer the question of what is the best an adversary can acheive and what the optimal adversarial attacks look like \cite{yi_derivation_2020}. We continue to investigate the classification tasks within the regime of the encoding-decoding paradigm. Though the works by Yi et al., are the major motivations for this work, the goal of this work is completely different.  We investigate the learning of  classification models without presence of adversaries, and the connection between the models' generalization performance and the mutual information of the dataset generating distribution. Our work is also distinctly from the information bottleneck framework where the goal is to learn a minimal sufficient representation $T$ from an input $X$ about its label $Y$ \cite{kolchinsky_nonlinear_2019,tezuka_information_2021,wang_pac-bayes_2021}, and usually a classifier needs to be trained on top of $T$ for classification.

\begin{figure}[!htb]
	\centering
	\includegraphics[width=\linewidth]{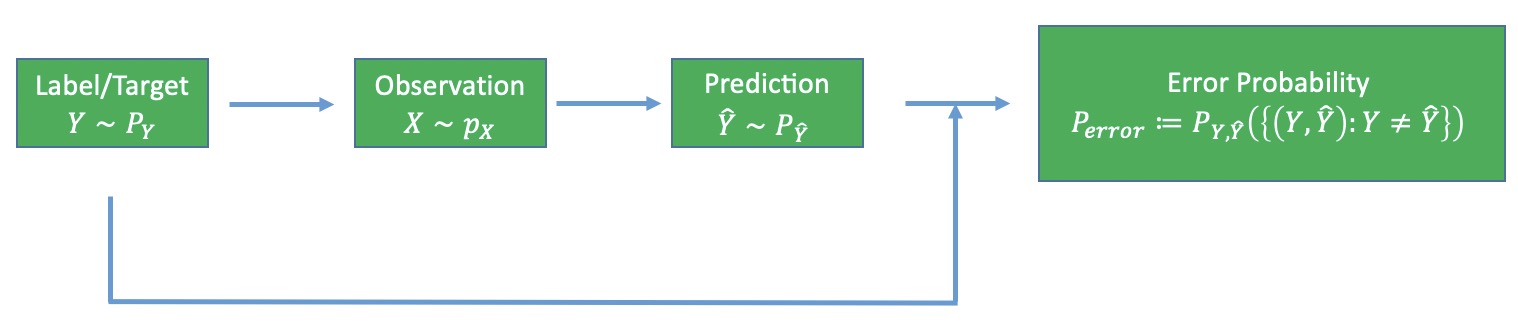}
	\caption{Information-theoretic view point of learning process.}\label{Fig:ErrorProbability}
\end{figure}

Our work is also highly related to that by \cite{mcallester_formal_2020} where the authors proposed a difference-of-entropy (doe) formulation for estimating the mutual information of a distribution from empirical observations from it \cite{mcallester_formal_2020}. In their formulation, two different neural networks are trained to learn the conditional entropy and the entropy, respectively. We use a formulation similar to the doe in \cite{mcallester_formal_2020}, but we only use a single neural network to learn both of them so that the parameters can be shared. Besides, we show that the existing cross entropy loss minimization approach actually learns the conditional entropy, and establish error probability lower bound in terms of the mutual information. 

\begin{table}[h]
	\centering
	\small
	\begin{tabular}{|l | c |}
		\hline 
		Loss obj. & Formula \\
		\hline 
		celLoss & $H(P_{Y|X}, Q_{Y|X})$\\
		\hline 
		celLoss$+$LSR & $(1-\epsilon)H(P_{Y|X}, Q_{Y|X}) + \epsilon H(U_{Y|X}, Q_{Y|X})$\\
		\hline celLoss$+$CP & $(1-\epsilon)H(P_{Y|X}, Q_{Y|X}) - \epsilon H(Q_{Y|X}, Q_{Y|X})$\\
		\hline 
		celLoss$+$LC & $(1-\epsilon)H(P_{Y|X}, Q_{Y|X}) + \epsilon H(Q_{Y|X}, Q_{Y|X})$\\
		\hline 
		milLoss & $H(P_{Y|X}, Q_{Y|X}) + \lambda_{ent} H(P_{Y}, Q_{Y})$\\
		\hline 
	\end{tabular}
	\caption{Comparisons among celLoss, LSR+celLoss, CP+celLoss, LC+celLoss, and regularized form of milLoss. The $P_{Y|X}$ and $P_Y$ are the conditional and marginal label distribution, respectively. The $Q_{Y|X}$ and the $U_{Y|X}$ are the predicted conditional label distribution and the uniform conditional label distribution.}\label{Tab:LossFunctionComparisons}
\end{table}

Another line of works which is highly related to this paper includes \cite{szegedy_rethinking_2016,meister_generalized_2020,pereyra_regularizing_2017,wang_proselflc_2021} where the regularized forms of celLoss are considered such as the LSR, CP, and LC. See the difference between these loss functions and the regularized form of our proposed one in Table \ref{Tab:LossFunctionComparisons}. The key assumption of the LSR and the CP is that the one-hot label is too confident, and a less confident prediction should be preferred. This is achieved by encouraging the prediction to be also close to a uniform distribution in LSR, or to have high entropy in CP \cite{szegedy_going_2015,pereyra_regularizing_2017}. The LC assumes the model will fit to the data distribution before overfitting to the noise during training, and the model should trust its prediction after certain stages during training. This is achieved by encouraging the model to have low-entropy or high-confidence prediction \cite{wang_proselflc_2021}. In both \cite{szegedy_rethinking_2016,pereyra_regularizing_2017,wang_proselflc_2021} and most of other related works, the regularizations still look at the conditional label distribution only while the regularization term in our formulation looks at the marginal label distribution.

{\bf Notations:}  We denote by $P(X)$ or $P_X$ the probability mass function of $X$ if $X$ is a discrete random variable or vector, and by $p(X)$ or $p_X$ the probability density function of $X$ if $X$ is continuous. Without loss of generality, we refer to both as probability distributions. The joint distribution of a continuou radnom variable $X$ and discrete random variable $Y$ will be denoted by $p_{X,Y}$ or $p(X,Y)$. For a distribution $Q_{Y}$ parameterized by $\theta$, we will denote it alternatively by $Q_{Y;\theta}$ and $Q_{Y}(\cdot;\theta)$. The probability mass (or the probability density) at a realization $y$ of discrete (continuous) random variable $Y$ will be denoted by $P_Y(y)$ and $P_Y(Y=y)$ (or $p_Y(y)$ and $p_Y(Y=y)$) interchangeablely. We use $[B]$  to denote a set $\{1,2,\cdots,B\}$ where $B$ is a positive integer. We denote by $\mcS:=\{(x_i,y_i)\}_{i=1}^N \subset\mbR^n\times[C]$ a set of data samples where $x_i\in\mbR^n$ is the input (or feature) and $y_i$ is the corresponding output (or label, or prediction, or target), and by $\mcS_x:=\{x_i\}_{i=1}^N$ a set of features or inputs by dropping the labels. In this paper, we assume $\|x_i - x_j\| >0, \forall i\neq j$. The $\mcS|_x$ denotes a subset of $\mcS$ with the input being $x$, i.e., $\mcS|_x :=\{(x_i,y_i)\in\mcS: x_i=x\}$. We denote by $\bm{0}$ a vector or a matrix whose elements are all zero, and by $\bm{1}$ a vector or a matrix whose elements are 1. We use $I_n\in\mbR^{n\times n}$ to denote an identity matrix. The $|Q|$ denotes the determinant of a square matrix $Q$, and the cardinality of a set $\mcS$ is denoted by $\card(\mcS)$ or $|\mcS|$. All the proofs can be found in the Appendix.

\section{Preliminaries}\label{Sec:Preliminaries}

We consider the classification task in machine learning, i.e., given a dataset $\mcS:=\{(x_i,y_i)\}_{i=1}^N $ drawn I.I.D. according to a joint data distribution $p_{X,Y}$ where $(x_i,y_i)\in\mbR^n\times[C]$ with $C$ being a positive integer, we want to learn a mapping $M:\mbR^n\to[C]$ from $\mcS$ such that $M$ can classify an unseen sample $x'\sim p_X$ in $\mbR^n$ to the correct class. The mutual information $I(X,Y)$ of the input $X$ and the label $Y$ under the joint distribution $p_{X,Y}$ can then be defined as
\begin{align}\label{Eq:MutualInfoDefn}
	I(X;Y)
	&:= \int_{\mbR^n} \sum_{y\in [C]} p(x,y) \log\left( \frac{p(x,y)}{p(x) P(y)} \right) dx,
\end{align} 
where we also define $p(x,y):=p(x)P(y|x)$ and $p(x,y):=P(y)p(x|y), \forall x\in\mbR^n, y\in[C]$. The entropy, differential entropy, and cross entropy are exactly the same as those in Shannon information theory. We define the conditional entropy, the conditional differential entropy, and the conditional cross entropy as in Definitions \ref{Definition:ConditionalEnt} and \ref{Definition:ConditionalCrossEntropy}, respectively.

\begin{definition}\label{Definition:ConditionalEnt}
	(Conditional Differential Entropy and Conditional Entropy) For a joint distribution $p_{X,Y}$ of a continuous random vector $X\in\mbR^n$ and discrete random variable $Y\in[C]$, we define the conditional differential entropy $h(X|Y)$ as 
	$h(X|Y):=  \sum_{y\in[C]} P(y) \int_{\mbR^n} p(x|y) \log\left( \frac{1}{p(x|y)} \right)dx$, 
	and the instance conditional differential entropy at realization $y$ for $Y$ as $
	h(X|y):= \int_{\mbR^n} p(x|y) \log\left( \frac{1}{p(x|y)} \right)dx$. We define the conditional entropy $H(Y|X)$ as
	$H(Y|X):=  \int_{\mbR^n}  p(x)   \sum_{y\in[C]} P(y|x) \log\left( \frac{1}{P(y|x)} \right)dx$, 
	and the instance conditional entropy at realization $x$ of $X$ as $
	H(Y|x):= \sum_{y\in[C]} P(y|x) \log\left( \frac{1}{P(y|x)} \right)$
\end{definition}

\begin{definition}\label{Definition:ConditionalCrossEntropy}
	(Conditional Cross Entropy) For two joint distributions $p_{X,Y}$ and $q_{X,Y}$ of a continuous random vector $X\in\mbR^n$ and discrete random variable $Y\in[C]$, we define the conditional cross entropies as 
	$H(P_{Y|X},Q_{Y|X})
	:= \int_{\mbR^n} p_X(x) \sum_{y\in[C]} P_{Y|X}(y|x) \log\left( \frac{1}{Q_{Y|X}(y|x)} \right)dx$, 
	and 
	$h(p_{X|Y}, q_{X|Y}):= \sum_{y\in[C]} P_Y(y)  \int_{\mbR^n} p_{X|Y}(x|y) \log\left( \frac{1}{q_{X|Y(x|y)}} \right)dx$.
\end{definition}

The conditional cross entropy will be used to derive a new formulation for training classifiers. The connections among these information-theoretic quantities associated with a hybrid distribution $p_{X,Y}\in\mbR^n\times[C]$ are presented in Theorem \ref{Thm:ITQuantitiesConnections}.

\begin{thm}\label{Thm:ITQuantitiesConnections}
	(Connections among Different Information-theoretic Quantities) With the definition of mutual information between a continous random vector $X\in\mbR^n$ and a discrete random variable $Y\in[C]$ in \eqref{Eq:MutualInfoDefn} and Defintions \ref{Definition:ConditionalEnt}-\ref{Definition:ConditionalCrossEntropy}, we have
	$I(X,Y) = H(Y) - H(Y|X)$ and $I(X,Y) = h(X) - h(X|Y).$
\end{thm}

\section{Traning Classifiers via Conditional Entropy Learning}\label{Sec:CrossEntropyLossTrainingAsConditionalEntropyEstimation}

The common practice in the machine/deep learning community separates the training or learning process and the decision or inference process, i.e.,  by first learning a conditional probability $Q_{Y|X}(y|x; \theta_{Y|X})\in[0,1]$of $Y$ given realization $x$ of $X$, and then making decisions about the labels via checking which class achieves the highest probability. Here, the $\theta_{Y|X}$ denotes the parameters associated with the function $Q_{Y|X}(\cdot;\theta_{Y|X})$. The former process is achieved by minimizing the cross entropy loss between an empirical conditional distribution $\hat{P}_{Y|X}$ and the estimated conditional distribution $Q_{Y|X}(\cdot;\theta_{Y|X})$, i.e.,
\begin{align}\label{Eq:CondEntLoss}
	\min_{\theta_{Y|X}} \frac{1}{N} \sum_{ (x,y)\in\mcS, c\in[C]} 
	\hat{P}_{Y|X}(c|x) \log\left( \frac{1}{ Q_{Y|X}(c|x; \theta_{Y|X})} \right) ,
\end{align}
while the inference process is then achieved via $\bar{y}_i = \arg\max_{c\in[C]} \left[Q_{Y|X}(x_i; \theta_{Y|X})\right]_c$. An example of $\hat{P}_{Y|X}(c|x)$ is the one-hot encoding \cite{he_deep_2015}, i.e., $\hat{P}_{Y|X}(c|x_i) = 1$ if $c=y_i$, or 0 if otherwise.

The empirical distribution $\hat{P}_{Y|X}$ is usually affected by the data collection process and the encoding methods for labels. For example, in image recognition tasks, a single image can have multiple objects, but it is at the human annotators' discretion about which label to use. Even for this particular label, different label encoding methods can give different label representations. The common approaches for encoding the $\hat{P}_{Y|X}$ include the one-hot representation, and the regularized forms such as the  $(1-\epsilon) \hat{P}_{Y|X} - \epsilon Q_{Y|X}$  \cite{pereyra_regularizing_2017,wang_proselflc_2021}. 

With the one-hot encoding for labels, the cross entropy minimization \eqref{Eq:CondEntLoss} can be simplified as
\begin{align}\label{Eq:CondEntLoss_OneHot}
	\min_{\theta_{Y|X}} \frac{1}{N} \times \sum_{(x,y)\in\mcS} \left( 
	\log\left( \frac{1}{ Q_{Y|X}(y|x; \theta_{Y|X})}  \right) \right)
\end{align}
The $\sum_{ c\in[C]} 
\hat{P}_{Y|X}(c|x) \log\left( \frac{1}{ Q_{Y|X}(c|x; \theta_{Y|X})} \right)$ in \eqref{Eq:CondEntLoss} can be interpreted as an estimate of the instance conditional entropy of $P_{Y|X}(Y|x)$ conditioning on the realization $x$ of $X$, i.e., $H(Y|x) \approx \sum_{ c\in[C]} 
\hat{P}_{Y|X}(c|x) \log\left( \frac{1}{ Q_{Y|X}(c|x;\theta_{Y|X}) } \right)$.
In Theorem \ref{Thm:EntEstViaCrossEnt}, we will show that this is indeed the case under certain conditions, and the cross entropy minimization problem in \eqref{Eq:CondEntLoss} is essentially learning the conditional entropy of the truth data distribution. With the uniqueness assumption of the realizations of $X$ in the dataset $\mcS$, the uniform distribution can be treated as the empirical distribution of $X$ over $\mcS_x$, and the objective function in \eqref{Eq:CondEntLoss} becomes an estimate of the conditional entropy $H(Y|X)$ with respect to $P_{Y|X}$, i.e., 
\begin{align}\label{Eq:CdntEtrp}
	& H(Y|X) \nonumber \\
	& \approx \sum_{(x,y)\in\mcS} 
	\left( \hat{P}_X(x) \hat{P}_{Y|X}(y|x) \log\left( \frac{1}{ Q_{Y|X}(y|x;\theta_{Y|X})} \right) \right).
\end{align}

\begin{thm}\label{Thm:EntEstViaCrossEnt}
	(Cross Entropy Minimization as Entropy Learning) For an arbitrary discrete distribution $P_{Y}$ in $[C]$, we have 
	$H(Y) \leq \inf_{Q_Y} H(P_Y,Q_Y)$, 
	where $Q_Y$ is a distribution of $Y$, and equality holds if and only if $P_Y=Q_Y$. When a set of $N$ data points $\mcS:=\{y_i\}_{i=1}^N$ drawn independently from $P_{Y}$ is given, by defining $R(y):= \frac{P_Y(y)}{\hat{P}_Y(y)}$ where $\hat{P}_Y$ is the empirical distribution associated with $\{y_i\}_{i=1}^N$, we have 
	$H(Y) \leq \inf_{Q^g_Y} H(\hat{P}_Y^g, Q_Y^g)$,
	where $\hat{P}_Y^g$ is defined as $
	\hat{P}_Y^g(y) : = \hat{P}_Y(y) R(y), \forall {y\in[C]}$, 
	and $Q_Y^g$ is defined as $
	Q_Y^g(y) = Q_Y(y)R(y), \forall {y\in[C]}$, 
	with $Q_Y$ being a distribution of $Y$. The inequality holds if and only if $Q_Y = P_Y = \hat{P}_Y, \forall {y\in[C]}$.
\end{thm}

Theorem \ref{Thm:EntEstViaCrossEnt} shows that the entropy $H(Y)$ is upper bounded by the cross entropy $H(P_Y,Q_Y)$. It also tells us that the entropy $H(Y)$ can actually be estimated from the empirical distribution $\hat{P}_Y$ over the sample set by minimizing $H(\hat{P}_Y^g,{Q}_Y^g)$. In practice, when we assume $R(y)=1, \forall y\in[C]$, and this gives $\hat{H}(Y)
:= \inf_{Q_Y} H(\hat{P}_Y,{Q}_Y)$
which is equal to $H(Y)$ if and only if  $P_Y(y) = \hat{P}_Y(y) = Q_Y(y)$. This essentially implies the possibility of learning entropy from empirical samples. 

In the classification tasks as we discussed previously, we can have similar formulation for $H(Y|X)$, i.e., $\hat{H}(Y|X):= \inf_{Q_{Y|X}} H(\hat{P}_{Y|X},{Q}_{Y|X})$. In \eqref{Eq:CondEntLoss}, the objective function in the cross entropy minimization is essentially $H(\hat{P}_{Y|X},{Q}_{Y|X})$ with ${Q}_{Y|X}$ parameterized by $\theta_{Y|X}$. This implies that the commonly used cross entropy loss minimization learns the conditional entropy $H(Y|X)$. We will refer to these classifiers as conditional entropy learning classifiers (CELC, /selk/). From Theorem \ref{Thm:ITQuantitiesConnections},  the mutual information (MI) can be learned via
\begin{align}\label{Eq:MutualInfoViaDoE_YEntMinusYCondEnt}
	I(X;Y)
	& = H(Y) - H(Y|X) \nonumber \\
	& \approx \inf_{Q_Y} H(\hat{P}_Y,Q_Y) - \inf_{Q_{Y|X}} H(\hat{P}_{Y|X},Q_{Y|X}).
\end{align}

\section{Training Classifiers via Mutual Information Learning}\label{Sec:TrainingClassifiersViaMutualInformationEstimation}

In this section, we formally present a new framework for training classifiers, i.e., via mutual information learning instead of the conditional entropy learning in existing paradigm, and we refer to classifier trained in this way as mutual information learned classifier (MILC, /milk/). 

For the mutual information learning formulation in \eqref{Eq:MutualInfoViaDoE_YEntMinusYCondEnt}, we can parameterize $Q_Y$ and $Q_{Y|X}$ via a group of parameters $\theta_{Y|X}\in\mbR^{m}$. More specifically, we parameterize $Q_{Y|X}(Y|X;\theta_{Y|X})$ using $\theta_{Y|X}$, and then calculate the marginal estimation $Q_{Y}(Y;\theta_{Y|X})$ via ${Q}_Y(y; \theta_{Y|X}) 
:= \sum_{x\in\mcS_x} \hat{P}_X(x) Q_{Y|X}(y|x;\theta_{Y|X})$.
Thus, the mutual information has the following form
\begin{align}\label{Eq:MIAlterEstParametericMethod_MargDistFromLearnedJointDist}
	I(X;Y)
	& \approx  \inf_{\theta_{Y|X}} 
	H(\hat{P}_Y,{Q}_Y(Y;{\theta_{Y|X}})) \nonumber\\ 
	&\quad - 
	\inf_{\theta_{Y|X}}H(\hat{P}_{Y|X}, Q_{Y|X}(Y|X; \theta_{Y|X})),
\end{align}
which is a multi-object optimization problem \cite{miettinen_nonlinear_2012}. An equivalent regularized form of \eqref{Eq:MIAlterEstParametericMethod_MargDistFromLearnedJointDist} can be as follows
\begin{align}\label{Eq:MIAlterEstParametericMethodEquivalentRegluarizedForm}
	\inf_{\theta_{Y|X}} 
	H(\hat{P}_{Y|X}, Q_{Y|X}) +  \lambda_{ent}H(\hat{P}_Y, Q_Y),
\end{align}
where $\lambda_{ent}>0$ is a regularization hyperparameter. In \eqref{Eq:MIAlterEstParametericMethodEquivalentRegluarizedForm}, the $H(\hat{P}_{Y|X}, Q_{Y|X})$ essentially corresponds to the cross entropy loss in multi-class classification while $H(\hat{P}_Y, Q_Y)$ can be treated as a regularization term. The cross entropy term guides machine learning algorithms to learn accurate estimation of condition entropy $H(Y|X)$, and the regularization term encourages the model also to learn the label entropy $H(Y)$. The overall deep neural network classifiers' trainning or learning pipeline, and the corresponding decision or inference pipeline are presented in the Appendix. To the best of our knowledge, the label entropy has never been used to train machine learning systems by the community. 

One may want to estimate the MI via empirical distribution only, i.e.,
\begin{align}\label{Eq:MutualInfoNaive}
	& I(X;Y) = H(Y) - H(Y|X) \nonumber\\
	& \approx \sum_{c\in[C]} \hat{P}_Y(c)\log\left(\frac{1}{\hat{P}_Y(c)}\right) \nonumber\\
	&\quad -  \sum_{i\in[N]} \hat{P}_X(x_i) \times \sum_{c\in[C]} \hat{P}_{Y|X}(c|x_i) \times \log\left(\frac{1}{\hat{P}_{Y|X}(c|x_i)}\right).
\end{align}
The problem is that the estimation can be quite inaccurate, and this is formally presented in Theorem \ref{Thm:EntEstBound} where we give the error bound of estimating entropy of $Y$ via the empirical distribution $\hat{p}_Y$. 

\begin{thm}\label{Thm:EntEstBound}
	(Error Bound of Entropy Learning from Empirical Distribution) For two arbitrary distributions $P_Y$ and $\hat{P}_Y$ of a discrete random variable $Y$ over $[C]$, we have 
	\begin{align*}
		\sum_{y\in[C]} R(y) \log\left( \frac{1}{{P}_Y(y)} \right)
		\leq 
		\Delta
		\leq  \sum_{y\in[C]} R(y) \log\left( \frac{1}{\hat{P}_Y(y)} \right)
	\end{align*}
	where $\Delta:=H_{P_Y}(Y) - H_{\hat{P}_Y}(Y)$, $R(y) = P_Y(y) - \hat{P}_Y(y), \forall y \in[C]$, and $H_{P_Y}(Y)$ is the entropy of $Y$ calculated via $P_Y$. The equality holds iff $P_Y = \hat{P}_Y$. 
\end{thm}

\begin{table*}[h]
	\centering
	\begin{tabular}{|l | c | c| c| c| c|}
		\hline 
		top-1 accuracy & celLoss & celLoss$+$LSR & celLoss$+$CP & celLoss$+$LC & milLoss \\
		\hline
		MLP & $0.934\pm0.000$ & 0.934$\pm$0.001 & 0.930$\pm$0.002 & 0.932$\pm$0.000 & $0.945\pm0.001$\\
		\hline 
		CNN & $0.981\pm0.000$ & 0.982$\pm$0.001 & 0.980$\pm$0.000 & 0.980$\pm$0.001 & $0.984\pm0.001$\\
		\hline 
	\end{tabular}
	\caption{Top-1 accuracy on MNIST dataset associated with different models which are trained with different loss objective function.}\label{Tab:ProofOfConcept_MNIST}
\end{table*}

\begin{table*}[h]
	\centering
	\begin{tabular}{|l | c | c| c| c| c|}
		\hline 
		top-1 accuracy & celLoss & celLoss$+$LSR & celLoss$+$CP & celLoss$+$LC & milLoss \\
		\hline
		GoogLeNet & $0.803\pm0.006$ & 0.766$\pm$0.004 & 0.784$\pm$0.006 & 0.791$\pm$0.002 & $0.866\pm0.000$\\
		\hline
		ResNet-18 & $0.732\pm0.002$ & 0.679$\pm$0.001 & 0.703$\pm$0.006 & 0.726$\pm$0.005 & $0.832\pm0.004$\\
		\hline
		MobileNetV2 &0.676$\pm$0.007 & 0.647$\pm$0.005 & 0.660$\pm$0.008 & 0.677$\pm$0.004 & 0.762$\pm$0.006\\
		\hline 
		EfficientNet-B0 & 0.524$\pm$0.013 & 0.510$\pm$0.006 & 0.503$\pm$0.008 & 0.524$\pm$0.009 & 0.682$\pm$0.006\\
		\hline 
		ShuffleNetV2 & 0.604$\pm$0.003 & 0.554$\pm$0.005 & 0.578$\pm$0.005 & 0.600$\pm$0.004 & 0.677$\pm$0.003\\
		\hline
	\end{tabular}
	\caption{Top-1 accuracy on CIFAR-10 dataset associated with different models which are trained with different loss objective function.}\label{Tab:ProofOfConcept_CIFAR10}
\end{table*}

\section{Error Probability Lower Bounds via Mutual Information}\label{Sec:ErrorProbLBound}

In this section, we establish the classification error probability bound in terms of mutual information. We follow Yi et al. to model the learning process as in Figure \ref{Fig:ErrorProbability} \cite{yi_trust_2019,xie_information-theoretic_2019,yi_derivation_2020,yi_towards_2021}, and assume there is a label distribution $P_Y$, and based on realizations from $P_Y$, we can generate a set of observations from $p_X$. Given the observations, we want to infer their labels. By combining the ground truth labels sampled from $P_Y$ and the predicted labels, we then calculate the error probability as $
P_{error} :=
P_{Y,\hat{Y}}\left( \{(Y,\hat{Y}): Y\neq \hat{Y}\} \right)$.
This learning process is consistent with practice. For example, in a dog-cat image classification tasks, we first have the concepts of the two classes, i.e., cat and dog. Then, we can generate observations of these labels/concepts, i.e.,  images of cat and dog by taking pictures of them, or simply drawing them, or using neural image synthesis \cite{goodfellow_deep_2016}. We then use these observations to train models, hoping that they will finally be able to predict the correct labels. Under this framework, we can show that the error probability associated with the learning process as shown in Figure \ref{Fig:ErrorProbability} can be bounded via mutual information $I(X;Y)$, and the results are formally presented in Theorem \ref{Thm:ErrorProbabilityMI_Bounds}.

\begin{thm}\label{Thm:ErrorProbabilityMI_Bounds}
	(Error Probability Bound via Mutual Information) Assume that the learning process $Y\to X\to\hat{Y}$ in Figure \ref{Fig:ErrorProbability} is a Markov chain where $Y\in[C]$, $X\in\mbR^n$, and $\hat{Y}\in[C]$, then for the prediction $\hat{Y}$ from an arbitrary learned model, we have $
	\max\left( 0,\frac{2+H(Y) - I(X;Y) - a}{4} \right)
	\leq P_{error}$
	where $a:=\sqrt{(H(Y) - I(X;Y) - 2)^2 + 4}$.
\end{thm}

From Theorem \ref{Thm:ErrorProbabilityMI_Bounds}, we can see that the $P_{error}$ lower bound will decrease when $I(X;Y)$ increases. This is consistent with our intuitions. For example, when the dependence between the observation $X$ and the label $Y$ gets stronger (larger mutual information), it will be easier to infer $Y$ from $X$, thus a smaller error probability can occur. In Figure \ref{Fig:ErrroProbability_and_MI}, we give illustrations of the relation between the error probability lower bound and the mutual information for a balanced underlying data distribution (see Appendix for error probability lower bound for an unbalanced data distribution). We assume uniform marginal distribution for the label. For the case with 100 classes, if the label and the input has zero mutual information, i.e., no dependency between them, we can only draw a random guess and get 0.99 error probability while the lower bound from Theorem \ref{Thm:ErrorProbabilityMI_Bounds} is about 0.9.

\begin{figure}[!htb]
	\centering
	\includegraphics[width=0.6\linewidth]{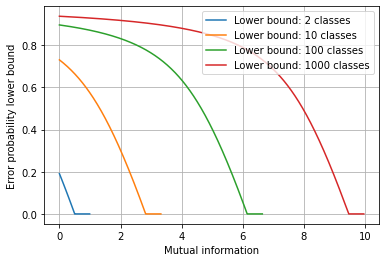}
	\caption{Error probability lower bound and mutual information for balanced dataset: uniform distirbution of labels is assumed.}\label{Fig:ErrroProbability_and_MI}
\end{figure}

\begin{figure}[!htb]
	\centering
	\begin{subfigure}[b]{0.22\textwidth}
		\centering
		\includegraphics[width=\linewidth]{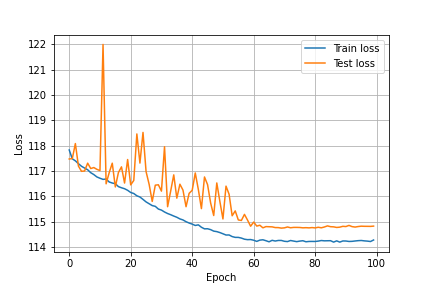}
		\caption{milLoss}\label{Fig:CIFAR10_GoogLeNet_milLoss}
	\end{subfigure}
	~
	\begin{subfigure}[b]{0.22\textwidth}
		\centering
		\includegraphics[width=\linewidth]{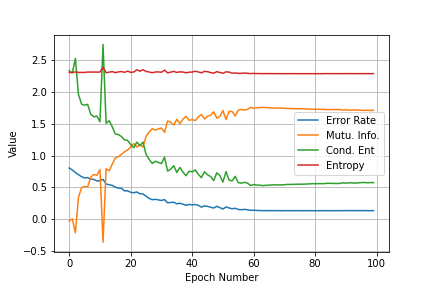}
		\caption{milLoss}\label{Fig:CIFAR10_GoogLeNet_milInformation}
	\end{subfigure}%
	\\
	\begin{subfigure}[b]{0.22\textwidth}
		\centering
		\includegraphics[width=\linewidth]{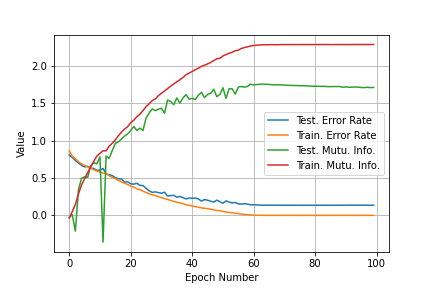}
		\caption{milLoss}\label{Fig:CIFAR10_GoogLeNet_milErrorRate_MI}
	\end{subfigure}
	~
	\begin{subfigure}[b]{0.22\textwidth}
		\centering
		\includegraphics[width=\linewidth]{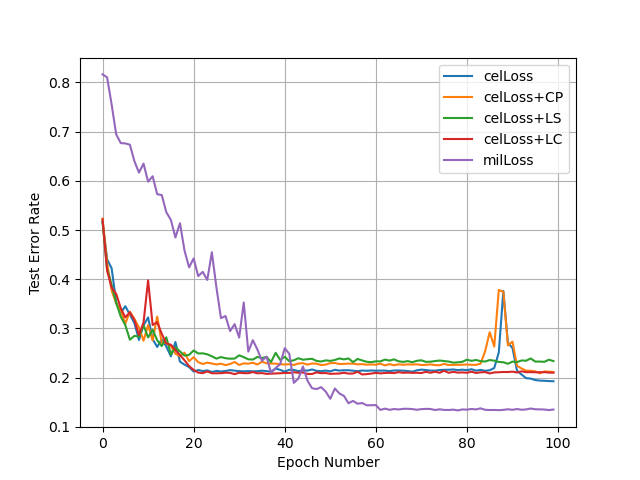}
		\caption{Error rates}\label{Fig:CIFAR10_GoogLeNet_milAcel_ErrorRate}
	\end{subfigure}
	\caption{GoogLeNet on CIFAR-10. \ref{Fig:CIFAR10_GoogLeNet_milLoss}: loss during training and testing at different epochs. \ref{Fig:CIFAR10_GoogLeNet_milInformation}: error rate, mutual information, label conditional entropy, and label entropy during test at different epochs. \ref{Fig:CIFAR10_GoogLeNet_milErrorRate_MI}: mutual information and error rate during training and testing at different epoch. \ref{Fig:CIFAR10_GoogLeNet_milAcel_ErrorRate}: testing error rate curves associated with different loss functions.}\label{Fig:CIFAR10_GoogLeNet}
\end{figure}

We now derive the mutual information bounds for a binary classification data model $P_{X,Y}$ in $\mbR^n\times\{-1,1\}$. In the data generation process, we first sample a label $y\in\{-1,1\}$, and then a corresponding feature $x$ from a Gaussian distribution, i.e., $P(Y=-1) = q, P(Y=1) = 1-q$, 
\begin{align}\label{Defn:BinaryClassificationDataModel}
	\small
	p(X=x|y) = \frac{1}{\sqrt{|2\pi\Sigma|}} \exp\left(-\frac{(x-y\mu)^T\Sigma^{-1}(x-y\mu)}{2}\right),
\end{align}
where $\mu\in\mbR^n$ is a mean vector, and $\Sigma\in\mbR^{n\times n}$ is a positive semidefinite matrix. In this data model, we can derive lower and upper bounds of the mutual information $I(X;Y)$, and the results are presented in Theorem \ref{Thm:BinaryClassificationDataModelTruthMI}. 

\begin{thm}\label{Thm:BinaryClassificationDataModelTruthMI}
	(Mutual Information Bounds of Binary Classification Dataset Model) For the data model with distribution defined in \eqref{Defn:BinaryClassificationDataModel}, we have the mutual information $I(X;Y)$ satisfying
	\begin{align*}\label{Eq:MIinBinaryClassificationModel}
		2\min(q,1-q)\mu^T\Sigma^{-1}\mu \leq I(X;Y) \leq 4q(1-q)\mu^T\Sigma^{-1}\mu.
	\end{align*}
\end{thm}

For a simplified case in $\mbR$ with $\mu=1$ and variance $\sigma^2=1$ for Theorem \ref{Thm:BinaryClassificationDataModelTruthMI}, when the variance becomes bigger, the two distributions $\mcN(1,\sigma^2)$ and $\mcN(-1,\sigma^2)$ get closer to each other. Thus, conditioning on $X$ can give very little information about $Y$, making it difficult to differentiate the two class labels. In Appendix, we give illustration of this point, and also the mutual information lower and upper bounds for the binary classification data model. Theorem \ref{Thm:BinaryClassificationDataModelTruthMI} can be easily generalized to multi-class classification in $\mbR^n$, and we leave this for future work. Based on Theorem \ref{Thm:EntEstBound} and \ref{Thm:BinaryClassificationDataModelTruthMI}, we can also derive a error probability lower bound for the binary classification data model  \eqref{Defn:BinaryClassificationDataModel} in Corollary \ref{Corola:ErrorProbabilityBound}.

\begin{cor}\label{Corola:ErrorProbabilityBound}
	For the data distribution defined in \eqref{Defn:BinaryClassificationDataModel}, we assume the $Y\to X\to \hat{Y}$ forms a Markov chain where $\hat{Y}$ is the prediction from a classifier, and we follow the learning process in Figure \ref{Fig:ErrorProbability} to learn the classifier. Then, the error probability for an arbitrary classifier must satisfy
	\begin{align*}
		\max\left( 0,\frac{2+H(Y) - 4q(1-q)\mu^T\Sigma^{-1}\mu - a}{4} \right)
		\leq P_{error},
	\end{align*}
	where $a:=\sqrt{(H(Y) - I(X;Y) - 2)^2 + 4}$.
\end{cor}

This is also intuitive. For example, when we consider the model in $\mbR$ with mean $\mu\in\mbR$ and $\sigma^2\in[0,+\infty)$, we have $\max\left( 0,\frac{2+H(Y) - 4q(1-q)\mu^T\Sigma^{-1}\mu - a}{4} \right)
\leq P_{error}$. When we increase $\mu$ (the distributions from two classes become farther from each other) and decrease $\sigma^2$ (the distributions from two classes become more concentrated), we are more likely to classify them correctly, thus a lower error probability. 

\begin{figure}[!htb]
	\centering
	\begin{subfigure}[b]{0.25\textwidth}
		\centering
		\includegraphics[width=\linewidth]{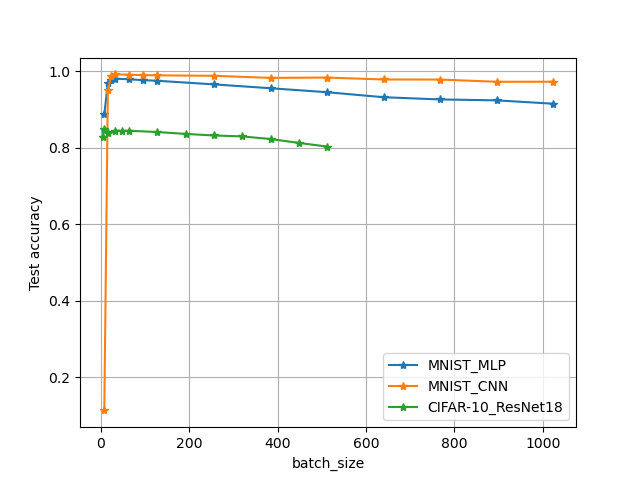}
		\caption{Batch size}
	\end{subfigure}%
	~
	\begin{subfigure}[b]{0.25\textwidth}
		\centering
		\includegraphics[width=\linewidth]{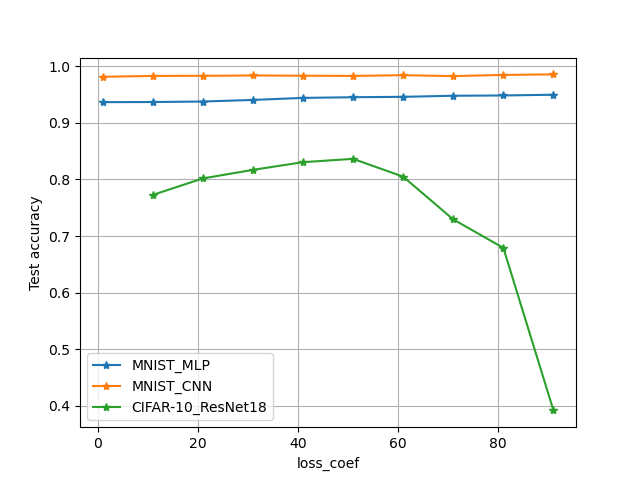}
		\caption{Loss coefficient $\lambda_{ent}$}
	\end{subfigure}
	\caption{Pinpoint batch size and $\lambda_{ent}$ parameters.}\label{Fig:ParamsPinpoint}
\end{figure}

\section{Experimental Results}\label{Sec:ExperimentalResults}

In this section, we present experimental results from multi-class classification on the MNIST and CIFAR-10 to validate our theory \cite{he_deep_2015}. 

For MNIST classification task, we use both a multiple layer perceptron (MLP) and a convolutional neural network (CNN) to train two different classifiers by using the regularized form of the mutual information learning loss in \eqref{Eq:MIAlterEstParametericMethodEquivalentRegluarizedForm}, the conditional entropy learning loss in \eqref{Eq:CondEntLoss}, and also its regularized forms as Table \ref{Tab:LossFunctionComparisons}. During training, we use a batch size of 512, and each model is trained for 77 epochs. For CIFAR-10 dataset, we use the ResNet-18, GoogLeNet, MobileNetV2, EfficientNetB0, and ShuffleNetV2 \cite{he_deep_2015,szegedy_going_2015,sandler_mobilenetv2_2019,tan_efficientnet_2020,ma_shufflenet_2018} to train classifiers. The batch size is set to be 256. Each model is trained for 100 epochs. For each model under each setup for MNIST, and CIFAR-10, the regularization parameter $\epsilon$ associated with the LSR, CP, and LC is fixed at 0.1 \cite{pereyra_regularizing_2017}. The $\lambda_{ent}$ takes value 5e1 when the mutual information learning loss is used to train the models. We do not use any data augmentations, nor do we use the weight decay. In the first set of experiments for MNIST and CIFAR-10, we conduct 3 trials for each model, and the reported results are averaged over the 3 trials. However, for later experiments, we conduct single trial for each model as we do not see much variations in the results across differential trials. We use SGD optimizer with a constant learning rate 1e-3 and a momentum 0.9.

The first set of experimental results associated with the baseline models for MNIST and CIFAR-10 datasets are presented in Table \ref{Tab:ProofOfConcept_MNIST} and \ref{Tab:ProofOfConcept_CIFAR10} where we present the testing data accuracy. From the results, we can see the proposed mutual information learning loss (milLoss) in \eqref{Eq:MIAlterEstParametericMethod_MargDistFromLearnedJointDist} achieved improvements of large margin when compared with the conditional entropy learning loss (celLoss) and its variants in \eqref{Eq:CondEntLoss}, e.g., from 0.52 to 0.68 when the EfficientNet-B0 is used. In fact, under our experiments setup, none of LSR, CP, and LC show any improvements in accuracy. 

We plot the learning curves in Figure \ref{Fig:CIFAR10_GoogLeNet}. More learning curves associated with other models can be found in the Appendix. From the results, we can see that the proposed approach can train classifiers with much better classification performances. We can also see a very strong connection between the mutual information and the error rate. The MILCs seems to take longer time to converge than the conditional entropy learned classifiers (CELCs), and we conjecture this is because learning the joint distribution $p_{X,Y}$ is more challenging than learning the conditional distribution $p_{Y|X}$. 

We conduct experiments with typical neural network architectures to investigate how the batch size and the $\lambda_{ent}$ affect the performance, the experimental setup except the batch size or $\lambda_{ent}$ is exactly the same as that of the baseline models. When evaluating effect of the batch size (or the $\lambda_{ent}$), we use fixed $\lambda_{ent}=5e1$ (or fixed batch size of 512 for MNIST and 256 for CIFAR-10). The results are presented in Figure \ref{Fig:ParamsPinpoint}. From the results we can see that for both MNIST and CIFAR-10 dataset, the testing accuracy does not always go up as the batch size increases, which is quite different what is expected for CELCs. When the batch size increases, both the signal pattern and the noise pattern will become stronger. The milLoss essentially learns the mutual information associated with the joint data generation distribution, and it can overfit to the noise pattern as the batch size increases since the mutual information itself encourage the model to consider the overall data generation distribution. This then results in the classification performance degradation. However, for CELCs, despite the stronger noise pattern caused by a larger batch size, the conditional entropy learning loss can help the model avoid overfitting to the noise pattern because it encourages the model to give high confidence prediction of labels. The cost is a less accurate characterization of the joint distribution $p_{X,Y}$ by CELCs.

\section{Conclusions}\label{Sec:Conclusions}

In this paper, we showed that the existing cross entropy loss minimization essentially learns the label conditional entropy. We proposed a new loss function which originating from mutual information learning, and established rigorous relation between the error probability associated with a model trained on a dataset and the mutual information of its underlying distribution. The application of our theory to a concrete binary classification data model was investigated. We also conducted extensive experiments to validate our theory, and the empirical results shows that the mutual information learned classifiers (MILCs) acheive far better generalization performance than those trained via cross entropy minimization or itr regularized variants.

\bibliography{202204_MILC}

\newpage
\appendix*
\section*{Appendix A: Definitions}\label{Sec:Dfnt}

\begin{definition}\label{Definition:MI}
	The mutual information $I(X,Y)$ of a continuous random vector $X\in\mbR^n$ and a discrete random variable $Y\in[C]$ under the joint distribution $p_{X,Y}$ is defined as
	\begin{align}
		I(X;Y)
		:= \int_{\mbR^n} \sum_{y\in [C]} p(x,y) \log\left( \frac{p(x,y)}{p(x) P(y)} \right) dx,
	\end{align} 
	where we also define $p(x,y):=p(x)P(y|x)$ and $p(x,y):=P(y)p(x|y), \forall x\in\mbR^n, y\in[C]$.
\end{definition}

\begin{definition}\label{Definition:Ent}
	(Differential Entropy and Entropy \cite{cover_elements_2012}) For a continuous random vector $X\in\mbR^n$ with distribution $p_X$, we define its differential entropy $h(X)$
	as
	\begin{align}\label{Defn:DifferentialEntropy}
		h(X): = - \int_{\mbR^n} p(x) \log(p(x)) dx.
	\end{align}
	For a discrete random variable $Y\in[C]$ with distribution $P_Y$, we define its entropy as
	\begin{align}\label{Defn:Entropy}
		H(Y):= -\sum_{y\in[C]} P(y) \log(P(y)).
	\end{align}
\end{definition}

\begin{definition}\label{Definition:ConditionalEnt}
	(Conditional Differential Entropy and Conditional Entropy) For a joint distribution $p_{X,Y}$ of a continuous random vector $X\in\mbR^n$ and a discrete random variable $Y\in[C]$, we define the conditional differential entropy $h(X|Y)$ as 
	\begin{align}\label{Defn:ConditionalDifferentialEnt}
		h(X|Y):=  \sum_{y\in[C]} P(y) \int_{\mbR^n} p(x|y) \log\left( \frac{1}{p(x|y)} \right)dx,
	\end{align}
	and the instance conditional differential entropy at realization $y$ for $Y$ as
	\begin{align}\label{Defn:InstcConditionalDifferentialEnt}
		h(X|y):= \int_{\mbR^n} p(x|y) \log\left( \frac{1}{p(x|y)} \right)dx.
	\end{align}
	
	We define the conditional entropy $H(Y|X)$ as
	\begin{align*}
		H(Y|X):=  \int_{\mbR^n}  p(x)   \sum_{y\in[C]} P(y|x) \log\left( \frac{1}{P(y|x)} \right)dx, 
	\end{align*}
	and the instance conditional entropy at realization $x$ of $X$ as 
	\begin{align}\label{Defn:InstcConditionalEnt}
		H(Y|x):= \sum_{y\in[C]} P(y|x) \log\left( \frac{1}{P(y|x)} \right).
	\end{align}
\end{definition}

\begin{definition}\label{Definition:CrossEntropy}
	(Cross Entropy) We define the cross entropy between two continuous distributions $p_X, q_X$ over the same continuous support set $\Omega$ as
	\begin{align} 
		h(p_X,q_X):=\int_\Omega p_X(x) \log\left( \frac{1}{q_X(x)} \right) dx.
	\end{align} 
	We define the cross entropy between discrete distributions $P_Y, Q_Y$ over the same discrete support set $\Omega$ as 
	\begin{align} 
		H(P_Y,Q_Y):=\sum_{y\in\Omega} P_Y(y) \log\left( \frac{1}{Q_Y(y)} \right).
	\end{align}
\end{definition}

\begin{definition}\label{Definition:ConditionalCrossEntropy}
	(Conditional Cross Entropy) For two joint distributions $p_{X,Y}$ and $q_{X,Y}$ of a continuous random vector $X\in\mbR^n$ and a discrete random variable $Y\in[C]$, we define the conditional cross entropy $H(P_{Y|X}, Q_{Y|X})$ as 
	\begin{align*}
		& H(P_{Y|X},Q_{Y|X})\\ 
		& := \int_{\mbR^n} p_X(x) \sum_{y\in[C]} P_{Y|X}(y|x) \log\left( \frac{1}{Q_{Y|X}(y|x)} \right)dx,
	\end{align*}
	and the conditional cross entropy $H(P_{X|Y}, Q_{X|Y})$ as 
	\begin{align*}
		& h(p_{X|Y}, q_{X|Y})\\
		& := \sum_{y\in[C]} P_Y(y)  \int_{\mbR^n} p_{X|Y}(x|y) \log\left( \frac{1}{q_{X|Y(x|y)}} \right)dx.
	\end{align*}
\end{definition}

\section*{Appendix B: Training and Inference of Mutual Information Learned Classifiers}\label{Sec:TrnInfrcPpln}

The training and inference pipeline of the proposed mutual information learning classifiers is shown in Figure \ref{Fig:MIML}. During the training process, we sample a data batch from training dataset $\mcS$ in each iteration, and then calculate the empirical marginal distribution $\hat{P}_X$, $\hat{P}_Y$ and $\hat{P}_{Y|X}$. The inputs $\{x_i\}_{i=1}^B$ will be fed to a machine learning system for it to learn the conditional distribution $Q_{Y|X;\theta_{Y|X}}$. We then combine the $Q_{Y|X;\theta_{Y|X}}$ with $\hat{P}_X$ and $\hat{P}_{Y|X}$ separately to calculate the learned marginal distribution $Q_{Y;\theta_{Y|X}}$ and the learned conditional entropy $\hat{H}(Y|X)$. The $Q_{Y;\theta_{Y|X}}$ is then combined with the $\hat{P}_Y$ to calculate the label entropy. We finally calculate the mutual information by subtracting conditional entropy from entropy. During inference, we feed an input $x$ to the model to get a conditional distribution, and the final class label prediction will be the one achieving the highest probability.

\begin{figure*}[!htb]
	\centering
	\begin{subfigure}[b]{\textwidth}
		\centering
		\includegraphics[width=\linewidth]{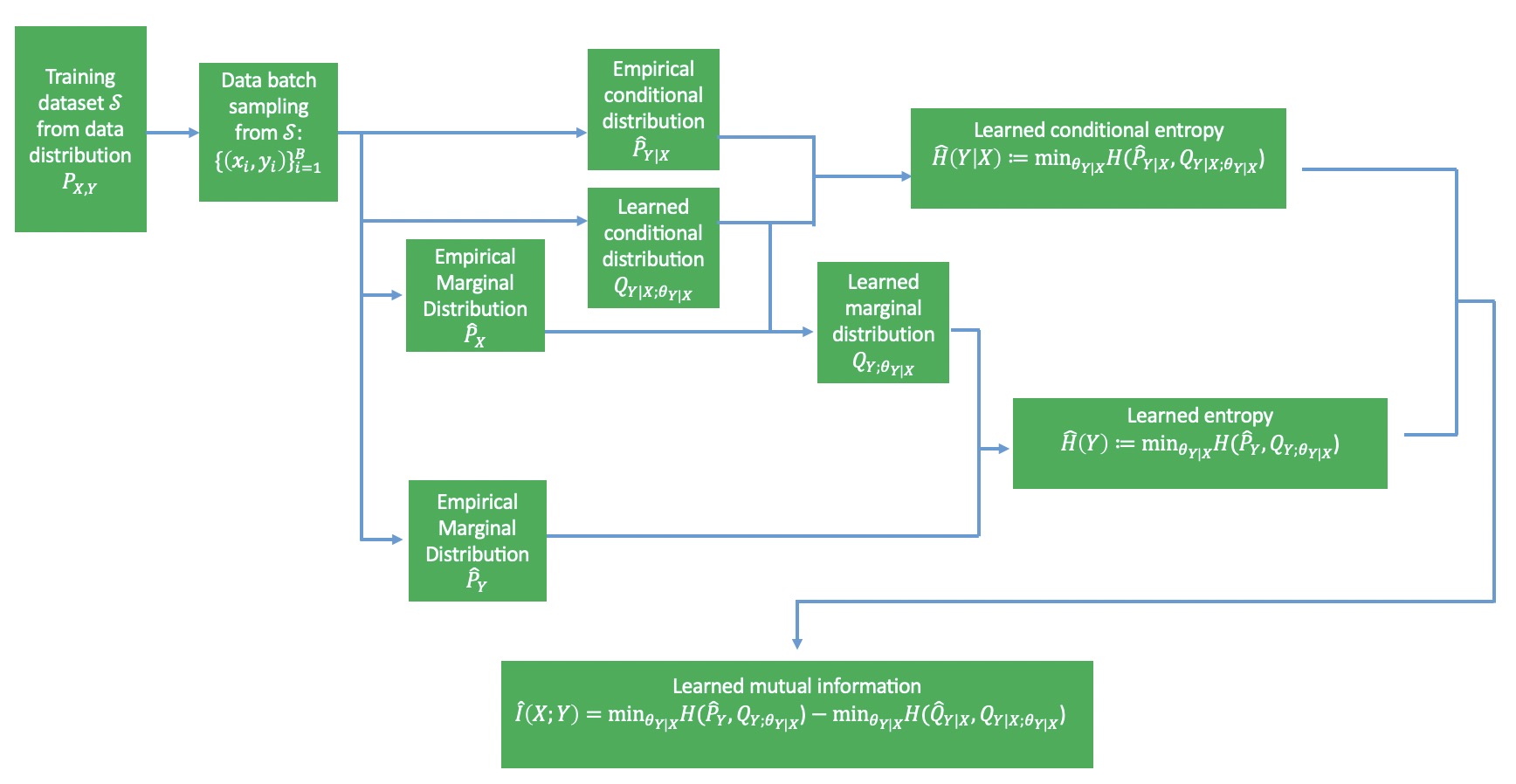}
		\caption{Training/learn pipeline}\label{Fig:MIML_TrainingPipelineV2}
	\end{subfigure}%
	\\
	\begin{subfigure}[b]{\textwidth}
		\centering
		\includegraphics[width=\linewidth]{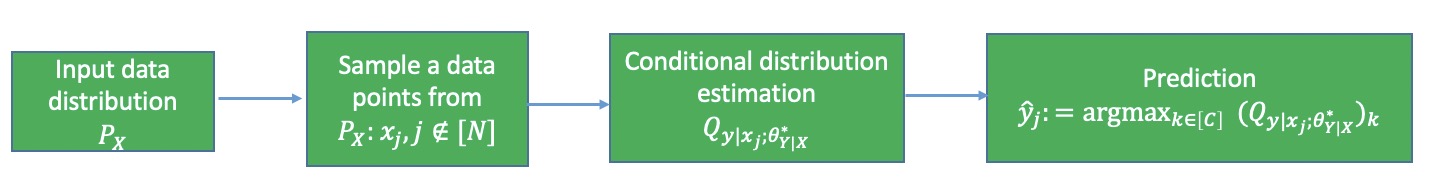}
		\caption{Inference/decision pipeline}\label{Fig:MIML_InferencePipeline}
	\end{subfigure}
	\caption{Mutual information learning classifiers (MILC)}\label{Fig:MIML}
\end{figure*}

\section*{Appendix C: Relations among Information-theoretic Quantities}\label{Sec:ITQttRlt}

\begin{thm}\label{Thm:ITQuantitiesConnections}
	(Connections among Different Information-theoretic Quantities) With the definition of mutual information between a continous random vector $X\in\mbR^n$ and a discrete random variable $Y\in[C]$ in \eqref{Eq:MutualInfoDefn} and Defintion \ref{Definition:Ent}-\ref{Definition:ConditionalCrossEntropy}, we have $I(X,Y) = H(Y) - H(Y|X)$, and $I(X,Y) = h(X) - h(X|Y)$.
\end{thm}

\begin{proof}\label{Proof:ITQuantitiesConnections}
	(of Theorem \ref{Thm:ITQuantitiesConnections}) From  \eqref{Eq:MutualInfoDefn}, we have
	\begin{align}\label{Eq:MI&ConditionalDifferentialEntrp}
		\scriptsize
		& I(X,Y)\\ 
		& = \int_{\mbR^n} p(x) \sum_{y\in[C]} P(y|x) \left( \log\left( \frac{p(x,y)}{P(y)} \right) -  \log\left( {p(x)} \right)\right) dx \nonumber \\
		& = \int_{\mbR^n} p(x) \sum_{y\in[C]} P(y|x) \log\left( \frac{p(x,y)}{P(y)} \right)dx \\
		&\quad -
		\int_{\mbR^n} p(x)\sum_{y\in[C]} P(y|x) \log\left( {p(x)} \right) dx \nonumber \\
		& = \int_{\mbR^n} \sum_{y\in[C]} p(x,y) \log\left( {p(x|y)} \right)dx 
		-
		\int_{\mbR^n} p(x) \log\left( {p(x)} \right) dx \nonumber \\
		& = - \int_{\mbR^n} \sum_{y\in[C]} P(y)p(x|y) \log\left( \frac{1}{p(x|y)} \right)dx + h(X) \nonumber \\
		& = h(X) - \sum_{y\in[C]} P(y) \int_{\mbR^n} p(x|y) \log\left( \frac{1}{p(x|y)} \right)dx  \nonumber \\
		& = h(X) - h(X|Y).
	\end{align}
	Similarly, we have from \eqref{Eq:MutualInfoDefn}
	\begin{align}
		\scriptsize
		& I(X,Y)\nonumber \\
		& = \sum_{y\in[C]}  P(y) \int_{\mbR^n} p(x|y) \left( \log\left( \frac{p(x,y)}{p(x)} \right) -  \log\left( {P(y)} \right)\right) dx \nonumber \\
		& = \sum_{y\in[C]}  P(y) \int_{\mbR^n} p(x|y) \log\left( \frac{p(x,y)}{p(x)} \right)dx \nonumber \\
		& \quad - 
		\sum_{y\in[C]}  P(y) \int_{\mbR^n} p(x|y) \log\left( {P(y)} \right)dx \nonumber \\
		& = \sum_{y\in[C]}  P(y) \int_{\mbR^n} p(x|y) \log\left( \frac{p(x,y)}{p(x)} \right)dx \nonumber \\
		& \quad - 
		\sum_{y\in[C]}  P(y)   \log\left( {P(y)} \right).
	\end{align}
	Thus, 
	\begin{align}
		\scriptsize
		& I(X,Y)\nonumber \\
		& = H(Y) - \sum_{y\in[C]}  P(y) \int_{\mbR^n} p(x|y) \log\left( \frac{1}{P(y|x)} \right)dx \nonumber \\
		& = H(Y) - \int_{\mbR^n}  \sum_{y\in[C]}  P(y) p(x|y) \log\left( \frac{1}{P(y|x)} \right)dx \nonumber \\
		& = H(Y) - \int_{\mbR^n}  \sum_{y\in[C]}  p(x) P(y|x) \log\left( \frac{1}{P(y|x)} \right)dx \nonumber \\
		& = H(Y) - \int_{\mbR^n}  p(x)   \sum_{y\in[C]} P(y|x) \log\left( \frac{1}{P(y|x)} \right)dx \nonumber \\
		& = H(Y) - H(Y|X).
	\end{align}
\end{proof}

\section*{Appendix D: Cross Entropy Minimization as Entropy Learning}\label{Sec:CrsEtrpMnmzt}

\begin{thm}\label{Thm:EntEstViaCrossEnt}
	(Cross Entropy Minimization as Entropy Learning) For an arbitrary discrete distribution $P_{Y}$ in $[C]$, we have 
	\begin{align}
		H(Y) \leq \inf_{Q_Y} H(P_Y,Q_Y),
	\end{align}
	where $Q_Y$ is a distribution of $Y$, and the equality holds if and only if $P_Y=Q_Y$. When a set of $N$ data points $\mcS:=\{y_i\}_{i=1}^N$ drawn independently from $P_{Y}$ is given, by defining $R(y):= \frac{P_Y(y)}{\hat{P}_Y(y)}$ where $\hat{P}_Y$ is the type or empirical distribution associated with $\{y_i\}_{i=1}^N$, we have 
	\begin{align}
		H(Y) \leq \inf_{Q^g_Y} H(\hat{P}_Y^g, Q_Y^g),
	\end{align}
	where $\hat{P}_Y^g$ is defined as
	\begin{align}
		\hat{P}_Y^g(y) : = \hat{P}_Y(y) R(y), \forall {y\in[C]},
	\end{align}
	and $Q_Y^g$ is defined as 
	\begin{align}
		Q_Y^g(y) = Q_Y(y)/R(y), \forall {y\in[C]},
	\end{align}
	with $Q_Y$ being a distribution of $Y$. The inequality holds if and only if $P_Y = \hat{P}_Y$ or $R(y)=1, \forall {y\in[C]}$.
\end{thm}

\begin{proof}\label{Proof:EntEstViaCrossEnt}
	(of Theorem \ref{Thm:EntEstViaCrossEnt}) From the definition of entropy, we have for an arbitrary distribution $Q_Y$ over $Y$
	\begin{align}\label{Eq:EntViaCrossEnt}
		H(Y) 
		& = \sum_{y\in[C]} P_Y(y) \log\left(\frac{1}{P_Y(y)}\right) \nonumber \\
		& = \sum_{y\in[C]} P_Y(y) \log\left(\frac{Q_Y(y)}{P_Y(y) }\frac{1}{Q_Y(y)}\right) \nonumber\\
		& = \sum_{y\in[C]} P_Y(y) \log\left(\frac{Q_Y(y)}{P_Y(y) }\right) \nonumber \\
		& \quad + \sum_y P_Y(y) \log\left(\frac{1}{Q_Y(y)}\right) \nonumber\\
		& = H(P_Y,Q_Y)  - D_{KL}(P_Y||Q_Y) \nonumber\\
		& \leq H(P_Y,Q_Y),
	\end{align}
	where $H(P_Y,Q_Y):= \sum_{y\in[C]} P_Y(y) \log\left(\frac{1}{Q_Y(y)}\right)$
	is the cross entropy between $P_Y$ and $Q_Y$, and  the $D_{KL}(P_Y||Q_Y) :=\sum_{y\in[C]} P_Y(y) \log\left(\frac{P_Y(y)}{Q_Y(y) }\right)$ 
	is the KL divergence between $P_Y$ and $Q_Y$, and we used the fact that $D_{KL}(P_Y||Q_Y)\geq0$. The equality holds iff $Q_Y(y) = P_Y(y), \forall {y\in[C]}$. The \eqref{Eq:EntViaCrossEnt} holds for arbitrary $Q_Y$, thus, $H(Y) = \inf_{Q_Y} H(P_Y,Q_Y)$
	
	For the cross entropy $H(P_Y,Q_Y)$, we have
	\begin{align}\label{Eq:CrossEntViaType}
		H(P_Y,Q_Y) 
		& =  \sum_{y\in[C]} P_Y(y) \log\left(\frac{1}{Q_Y(y)}\right) \nonumber \\
		& =  \sum_{y\in[C]} P_Y(y) \log\left(\frac{1}{Q_Y(y)} \frac{P_Y(y)}{\hat{P}_Y(y)} \frac{\hat{P}_Y(y)}{{P}_Y(y)}\right) \nonumber  \\
		& =  \sum_{y\in[C]} P_Y(y) \log\left(\frac{1}{Q_Y(y)} \frac{P_Y(y)}{\hat{P}_Y(y)} \right) \nonumber \\
		& + \sum_{y\in[C]} P_Y(y) \log\left( \frac{\hat{P}_Y(y)}{{P}_Y(y)}\right) \nonumber  \\ 
		& = - \sum_{y\in[C]} P_Y(y) \log\left(\frac{P_Y(y)}{\hat{P}_Y(y)} \right) \nonumber\\
		& + \sum_{y\in[C]} P_Y(y) \log\left(\frac{1}{Q_Y(y)} \frac{{P}_Y(y)}{\hat{P}_Y(y)}\right) \\
		& = -D_{KL}(P_Y||\hat{P}_Y) 
		+ \sum_{y\in[C]} P_Y(y) \log\left(\frac{1}{Q_Y(y)} \frac{{P}_Y(y)}{\hat{P}_Y(y)}\right)
	\end{align}
	Thus
	\begin{align}\label{Eq:CrossEntViaType2}
		H(P_Y,Q_Y) 
		& \leq \sum_{y\in[C]} \hat{P}_Y(y) \frac{P_Y(y)}{\hat{P}_Y(y)}  \log\left(\frac{1}{Q_Y(y)} \frac{{P}_Y(y)}{\hat{P}_Y(y)}\right) \nonumber \\
		& \leq \sum_{y\in[C]} \hat{P}_Y(y) R(y)  \log\left(\frac{1}{Q_Y(y) /R(y)}\right)\\
		&\leq \sum_{y\in[C]} \hat{P}_Y^g(y) \log\left(\frac{1}{Q_Y^g(y)}\right) \nonumber\\
		& = H(\hat{P}_Y^g,Q_Y^g),
	\end{align}
	where we used $\hat{P}_Y^g(y) = \hat{P}_Y(y) R(y)$, and $Q_Y^g(y) = Q_Y(y)/R(y)$, and $R(y) =\frac{{P}_Y(x)}{\hat{P}_Y(y)}$. The equality holds iff $P_Y(y) = \hat{P}_Y(y), \forall {y\in[C]}$ which implies that $D_{KL}(P_Y||\hat{P}_Y)=0$ and $R(y)=1$. Thus, $H(Y)=H(P_Y,Q_Y)=H(P_Y^g,Q_Y^g)$ if and only if $P_Y=Q_Y=\hat{P}_Y$. Thus, 
	\begin{align}\label{Eq:EntViaCrossEntCont1}
		H(Y)
		= \inf_{Q_Y} H(P_Y,Q_Y)
		= \inf_{Q^g_Y} H(\hat{P}_Y^g,{Q}_Y^g).
	\end{align}
	
\end{proof}

\section*{Appendix E: Error Bound of Entropy Learning from Empirical Distribution}\label{Sec:ErrBd_EtrpLrn_EmprcDstrbt}

\begin{thm}\label{Thm:EntEstBound}
	(Error Bound of Entropy Learning from Empirical Distribution) For two arbitrary distributions $P_Y$ and $\hat{P}_Y$ of a discrete random variable $Y$ over $[C]$, we have 
	\begin{align}
		\sum_{y\in[C]} R(y) \log\left( \frac{1}{{P}_Y(y)} \right)
		\leq 
		\Delta
		\leq  \sum_{y\in[C]} R(y) \log\left( \frac{1}{\hat{P}_Y(y)} \right)
	\end{align}
	where $\Delta:=H_{P_Y}(Y) - H_{\hat{P}_Y}(Y)$, $R(y) = P_Y(y) - \hat{P}_Y(y), \forall y \in[C]$, and $H_{P_Y}(Y)$ is the entropy of $Y$ calculated via $P_Y$. The equality holds iff $P_Y = \hat{P}_Y$. 
\end{thm}

\begin{proof}
	(of Theorem \ref{Thm:EntEstBound}) From the definition of entropy, we have 
	\begin{align*}
		& H_{P_Y}(Y) - H_{\hat{P}_Y}(Y) \nonumber \\
		& = \sum_{y\in[C]} P_Y(y)\log\left(\frac{1}{{P}_Y(y)}\right)
		- \sum_{y\in[C]} \hat{P}_Y(y)\log\left(\frac{1}{\hat{P}_Y(y)}\right) \nonumber \\
		& = \sum_{y\in[C]} P_Y(y)\log\left(\frac{\hat{P}_Y(y)}{{P}_Y(y)} \frac{1}{\hat{P}_Y(y)} \right) \nonumber \\
		&\quad - \sum_{y\in[C]}  \hat{P}_Y(y)\log\left(\frac{1}{\hat{P}_Y(y)}\right) \nonumber \\
		& = - \sum_{y\in[C]} P_Y(y)\log\left(\frac{{P}_Y(y)}{\hat{P}_Y(y)} \right) \nonumber \\
		&\quad
		+ \sum_{y\in[C]} P_Y(y)\log\left(\frac{1}{\hat{P}_Y(y)} \right)  - \sum_{y\in[C]}  \hat{P}_Y(y)\log\left(\frac{1}{\hat{P}_Y(y)}\right) \nonumber \\
		& = -D_{KL}(p_Y||\hat{P}_Y) \nonumber \\
		& \quad + \sum_{y\in[C]} P_Y(y)\log\left(\frac{1}{\hat{P}_Y(y)} \right)  - \sum_{y\in[C]}  \hat{P}_Y(y)\log\left(\frac{1}{\hat{P}_Y(y)}\right), \nonumber \\
		& \leq \sum_{y\in[C]} R(y) \log\left( \frac{1}{\hat{P}_Y(y)} \right).
	\end{align*}
	The equality holds if and only if $\hat{P}_Y=P_Y$. Similarly, we can show $\sum_{y\in[C]} R(y) \log\left( \frac{1}{{P}_Y(y)} \right)
	\leq 
	H_{P_Y}(Y) - H_{\hat{P}_Y}(Y)$.
\end{proof}

\section*{Appendix F: Classification Error Probabiliy Bound via Mutual Information}\label{Sec:Cls_ErrPrbBd_MI}

\begin{lemma}\label{Lem:ErrorEntropyUB}
	(\cite{yi_derivation_2020}) For arbitrary $x\in[0,1]$, we have
	\begin{align}
		x\log\left(\frac{1}{x}\right) + (1-x) \log\left(\frac{1}{1-x}\right)  \leq 1-2(x - 0.5)^2.
	\end{align}
\end{lemma}

Lemma \ref{Lem:ErrorEntropyUB} can be used to bound the entropy associated with a binary distribution, and a simple visual illustration of it is presented in Figure \ref{Fig:BinaryDistrEntUB} where we let $x$ be the error probability and $1-x$ be the correct probability (or accuracy).

\begin{figure}[!htb]
	\centering
	\includegraphics[width=\linewidth]{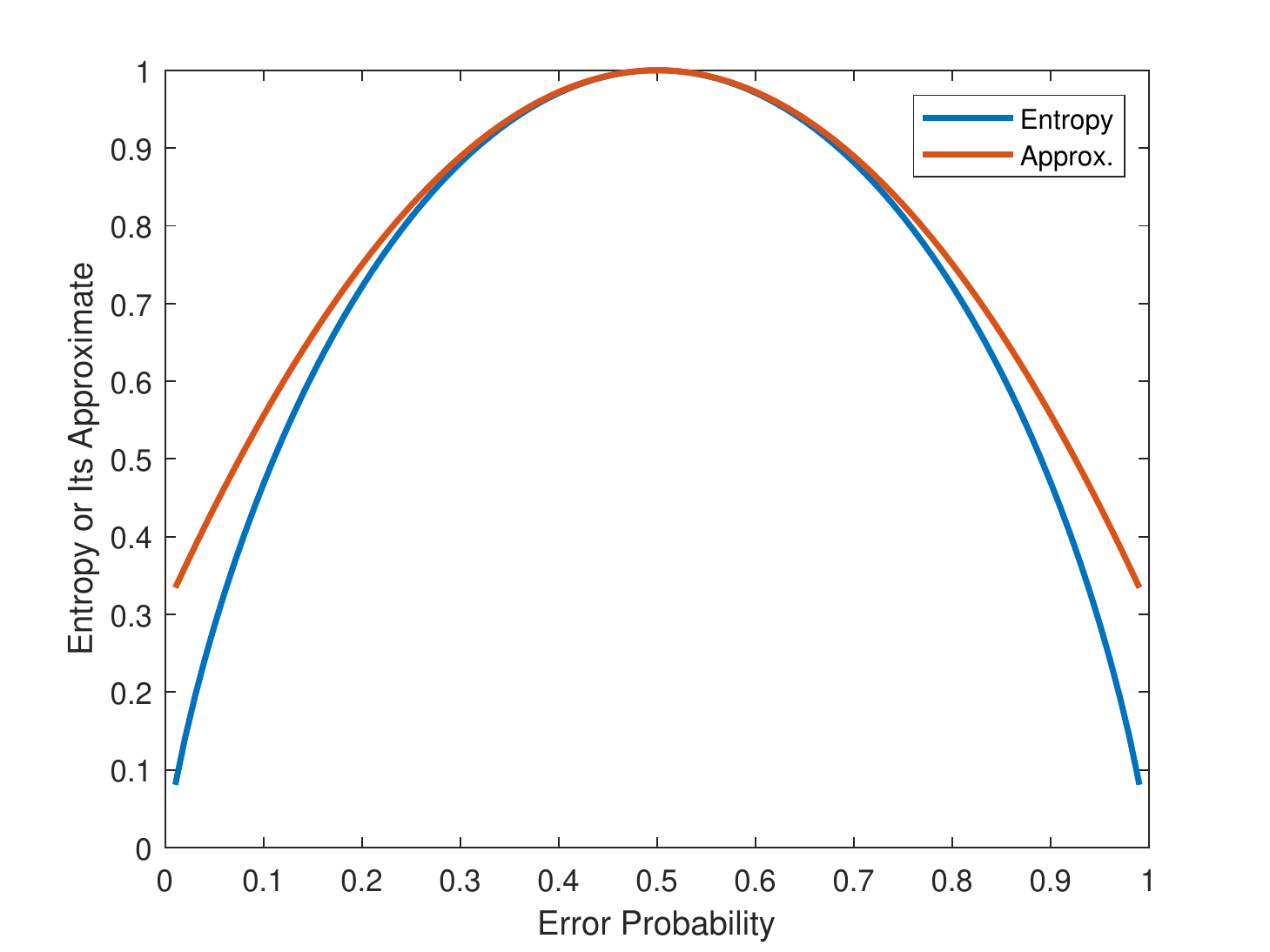}
	\caption{Upper bound of entropy assoicated with binary distribution: error probability $x$ and correct probability $1-x$.}\label{Fig:BinaryDistrEntUB}
\end{figure}

\begin{figure}[!htb]
	\centering
	\includegraphics[width=\linewidth]{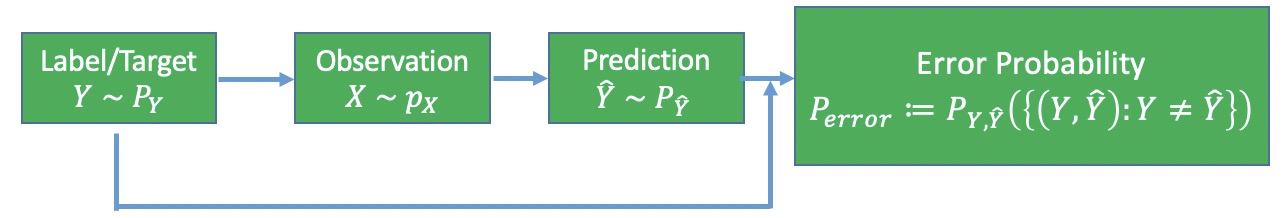}
	\caption{Information-theoretic view point of learning process.}\label{Fig:ErrorProbability}
\end{figure}

\begin{thm}\label{Thm:ErrorProbabilityMI_Bounds}
	(Error Probability Bound via Mutual Information) Assume that the learning process $Y\to X\to\hat{Y}$ in Figure \ref{Fig:ErrorProbability} is a Markov chain where $Y\in[C]$, $X\in\mbR^n$, and $\hat{Y}\in[C]$, then for the prediction $\hat{Y}$ from an arbitrary learned model, we have
	\begin{align}\label{Eq:ErrorProbabilityBounds}
		\max\left( 0,\frac{2+H(Y) - I(X;Y) - a}{4} \right)
		\leq P_{error}
	\end{align}
	where $a:=\sqrt{(H(Y) - I(X;Y) - 2)^2 + 4}$.
\end{thm}

\begin{proof}\label{Proof:ErrorProbabilityMI_Bounds}
	(of Theorem \ref{Thm:ErrorProbabilityMI_Bounds}) We define a random variable $E$
	\begin{align}
		E = \begin{cases}
			0, \text{if } Y= \hat{Y},\\
			1, \text{if } Y\neq \hat{Y},
		\end{cases}
	\end{align}
	then the error probability will become
	\begin{align}
		P_{error} = P_{E}(E=1).
	\end{align}
	From the properties of conditional joint entropy, we have
	\begin{align}\label{Eq:UpperBound}
		H(E,Y|\hat{Y})
		& = H(E|\hat{Y})  + H(Y|E,\hat{Y}) \nonumber \\
		& = H(E|\hat{Y}) + P_E(E=0) H(Y|\hat{Y},E=0) \nonumber\\
		& \quad + P_E(E=1) H(Y|\hat{Y},E=1) \nonumber\\
		& = H(E|\hat{Y}) + P_{E}(E=1)H(Y|\hat{Y},E=1) \nonumber\\
		& \leq H(E|\hat{Y}) + P_E(E=1) H(Y|\hat{Y})  \nonumber\\
		& \leq H(E) + P_E(E=1) H(Y|\hat{Y})  \nonumber\\
		& = H(P_{error}) + P_{error}H(Y|\hat{Y}),
	\end{align}
	where we used the fact that $H(Y|\hat{Y},E=0)=0$, $H(Y|\hat{Y},E=1)\leq H(Y|\hat{Y})$, and $H(E|\hat{Y}) \leq H(E)=H(P_{error})$ with $H(P_{error})$ defined as follows
	\begin{align*}
		H(P_{error})
		& :=	P_{error}\log\left(\frac{1}{P_{error}}\right) \\
		&\quad + (1-P_{error})\log\left( \frac{1}{1-P_{error}}\right).
	\end{align*}
	We also have
	\begin{align}
		H(E,Y|\hat{Y})
		& = H(Y|\hat{Y}) + H(E|Y,\hat{Y}) \nonumber\\
		& = H(Y|\hat{Y})  \label{Eq:HYcYhat}\\
		& = H(Y) - I(Y;\hat{Y})  \nonumber\\
		& \geq H(Y) - I(X;Y), \label{Eq:relaxed}
	\end{align}
	where we used the fact that $H(E|Y,\hat{Y})=0$, and the data processing inequality associated with Markov process $Y\to X\to \hat{Y}$, i.e., 
	\begin{align}\label{Eq:DataProcessingIneq}
		I(X;Y) \geq I(Y;\hat{Y}).
	\end{align}
	Thus, from \eqref{Eq:UpperBound} and \eqref{Eq:HYcYhat}, we have
	\begin{align}
		H(P_{error}) + P_{error}H(Y|\hat{Y}) \geq H(Y|\hat{Y}).
	\end{align}
	Combining the above and \eqref{Eq:relaxed}, we get
	\begin{align}\label{Eq:IntermediateBound}
		H(P_{error})  
		\geq (1-P_{error}) \left(H(Y) - I(X;Y) \right),
	\end{align}
	which implies $P_{error} \geq 1 - \frac{H(P_{error})}{H(Y) - I(X;Y)}$.
	
	From Lemma \ref{Lem:ErrorEntropyUB}, we have
	\begin{align*}
		(1-P_{error}) \left(H(Y) - I(X;Y) \right) \leq 1-2(P_{error} - 0.5)^2,
	\end{align*}
	or
	\begin{align}
		a_2P_{error}^2 + a_1 P_{error} + a_0 \leq 0,
	\end{align}
	where 
	\begin{align*} 
		a_2:=2, \\
		a_1:=- (2+H(Y) - I(X;Y)),\\ 
		a_0:=H(Y) - I(X;Y)-0.5.
	\end{align*}
	By solving the above inequality for $P_{error}$, we get $P_{error} \in [\frac{2+H(Y) - I(X;Y) - a}{4}, \frac{2+H(Y) - I(X;Y) + a}{4}]$, 
	where $a$ is defined as 
	\begin{align*}
		a:=\sqrt{(H(Y) - I(X;Y) - 2)^2 + 4}.
	\end{align*}
	Since $P_{error}\in[0,1]$, we have 
	\begin{align}
		\max\left( 0,\frac{2+H(Y) - I(X;Y) - a}{4} \right)
		\leq P_{error}.
	\end{align}
	
\end{proof}

\section*{Appendix G: Mutual Information Bounds in a Binary Classification Data Model}\label{Sec:MIBdsBnrCls}

We consider a binary classification data model $P_{X,Y}$ in $\mbR^n\times\{-1,1\}$. In the data generation process, we first sample a label $y\in\{-1,1\}$, and then a corresponding feature $x$ from a Gaussian distribution. We model the feature as a Gaussian random vector $X$ with sample space $\mbR^n$, i.e., 
\begin{align}\label{Defn:BinaryClassificationDataModel}
	P(Y=-1) = q, P(Y=1) = 1-q, \nonumber \\
	p(X=x|y) = \frac{1}{\sqrt{|2\pi\Sigma|}} \exp\left(-\frac{(x-y\mu)^T\Sigma^{-1}(x-y\mu)}{2}\right),
\end{align}
where $\mu\in\mbR^n$ is a mean vector, and $\Sigma\in\mbR^{n\times n}$ is a positive semidefinite matrix. We will denote by $Y\sim \mcB(q)$ the marginal distirbution $P_Y$, and $X\sim \mcN(y\mu,\Sigma)$ the conditional distribution $p_{X|y}$. 

\begin{lemma}\label{Lem:QuadraticFormExpectation}
	(Expectation of Quadratic Form of Gaussian Random Vector) For a Gaussian random vector $X\in\mbR^n$ following $\mcN(\mu,\Sigma)$, we have
	\begin{align}\label{Eq:QuadraticFormExpectation}
		\mbE_{X}\left[ X^TAX \right] = \tr(A\Sigma) + \mu^T\Sigma^{-1}\mu, \nonumber\\
		\mbE_{X}\left[ (X-\mu)^TA(X-\mu) \right] = \tr(A\Sigma), \nonumber\\
		\mbE_{X}\left[ (X+\mu)^T A (X+\mu) \right] = \tr(A\Sigma) + 4\mu^T\Sigma \mu,
	\end{align}
	where $A\in\mbR^{n\times n}$ is a square matrix.
\end{lemma}

\begin{proof}\label{Proof:QuadraticFormExpectation}
	(of Lemma \ref{Lem:QuadraticFormExpectation}) From the definition of expectation, we have
	\begin{align*}
		& \mbE_{X}\left[ X^TAX \right] \\
		& = \int_{\mbR^n}  \frac{1}{\sqrt{|2\pi\Sigma|}} \exp\left(-\frac{(x-\mu)^T\Sigma^{-1}(x-\mu)}{2}\right) x^TAx  dx \\
		& = \int_{\mbR^n}  \frac{1}{\sqrt{|2\pi\Sigma|}} \exp\left(-\frac{(x-\mu)^T\Sigma^{-1}(x-\mu)}{2}\right) \tr(Axx^T)  dx \\
		& = \int_{\mbR^n}  \frac{1}{\sqrt{|2\pi\Sigma|}} \exp\left(-\frac{(x-\mu)^T\Sigma^{-1}(x-\mu)}{2}\right) \\
		&\quad \times \tr\left(A(x-\mu)(x-\mu)^T - A\mu\mu^T + Ax\mu^T + A\mu x^T \right)  dx \\
		& = \int_{\mbR^n}  \frac{1}{\sqrt{|2\pi\Sigma|}} \exp\left(-\frac{(x-\mu)^T\Sigma^{-1}(x-\mu)}{2}\right) \\
		& \quad \times \left( \tr(A(x-\mu)(x-\mu)^T) - \mu^TA\mu + \mu^TAx + x^TA\mu \right)  dx \\
		& = \int_{\mbR^n}  \frac{1}{\sqrt{|2\pi\Sigma|}} \exp\left(-\frac{(x-\mu)^T\Sigma^{-1}(x-\mu)}{2}\right) \\
		&\quad \times \tr(A(x-\mu)(x-\mu)^T)  dx \\
		&\quad + \mu^TA\mu\\
		& = \tr(A\Sigma) + \mu^TA\mu.
	\end{align*}
	Similarly, we can derive the other two equations in \eqref{Eq:QuadraticFormExpectation}.
\end{proof}

\begin{thm}\label{Thm:BinaryClassificationDataModelTruthMIApp}
	(Mutual Information of Binary Classification Dataset Model) For the data model with distribution defined in \eqref{Defn:BinaryClassificationDataModelApp}, we have the mutual information $I(X;Y)$ satisfying
	\begin{align}\label{Eq:MIinBinaryClassificationModelApp}
		2\min(q,1-q)\mu^T\Sigma^{-1}\mu \leq I(X;Y) \leq 4q(1-q)\mu^T\Sigma^{-1}\mu.
	\end{align}
\end{thm}

\begin{proof}\label{Proof:BinaryClassificationDataModelTruthMI}
	(of Theorem \ref{Thm:BinaryClassificationDataModelTruthMIApp}) From the definition of mutual information, we have
	\begin{align}\label{Defn:MIinBinaryClassificationModel}
		I(X,Y) := h(X) - h(X|Y), 
	\end{align}
	where the differential entropy $h(X|Y=1) = h(X|Y=-1) = \frac{1}{2}\log(|2\pi e\Sigma|)$ and $e$ is the natural number. Thus 
	\begin{align}\label{Eq:ConditionalEntropy_XcY}
		h(X|Y) 
		& = P(Y=1)\times h(X|Y=1) \\
		&\quad + P(Y=-1)\times h(X|Y=-1) \\
		& =\frac{1}{2}\log(|2\pi e\Sigma|).
	\end{align}
	
	From the definition of data model in \eqref{Defn:BinaryClassificationDataModel}, we have the marginal distirbution $p_X$
	\begin{align}\label{Eq:XMarginalDistribution}
		p_X(x) 
		& = P_Y(Y=1) \times p_{X|Y=1}(X=x) \\
		&\quad + P_Y(Y=-1) \times p_{X|Y=-1}(X=x) \nonumber\\
		& =  \frac{(1-q)}{\sqrt{|2\pi\Sigma|}} \exp\left(-\frac{(x-\mu)^T\Sigma^{-1}(x-\mu)}{2}\right) \\
		&\quad +  \frac{q}{\sqrt{|2\pi\Sigma|}} \exp\left(-\frac{(x+\mu)^T\Sigma^{-1}(x+\mu)}{2}\right).
	\end{align}
	we have \eqref{Eq:XEntropy}, 
	\begin{align}\label{Eq:XEntropy}
		& h(X) \nonumber\\
		& := \int_\mbR p_X(x) \log\left( \frac{1}{p_X(x)} \right) dx \nonumber \\
	& =  \int_\mbR - p_X(x)  \nonumber \\
	& \quad \times \log\left( 
	\frac{(1-q)}{\sqrt{|2\pi\Sigma|}} \exp\left(-\frac{(x-\mu)^T\Sigma^{-1}(x-\mu)}{2}\right)
	\right) dx \nonumber \\
	&\quad + \int_\mbR -p_X(x)\nonumber \\
	& \quad \times   \log\left( 
	\frac{q}{\sqrt{|2\pi\Sigma|}} \exp\left(-\frac{(x+\mu)^T\Sigma^{-1}(x+\mu)}{2}\right)
	\right) dx
\end{align}

From Jensen's inequality for convex function $\exp(x)$, we have \eqref{Eq:LogOfConvexCombinationOfExp}. 

\begin{table*}
	\centering 
	\begin{minipage}{0.75\textheight}
		\begin{align}\label{Eq:LogOfConvexCombinationOfExp}
			& \log\left( 
			\frac{(1-q)}{\sqrt{|2\pi\Sigma|}} \exp\left(-\frac{ (x-\mu)^T\Sigma^{-1}(x-\mu) }{2}\right)
			+  \frac{q}{\sqrt{|2\pi\Sigma|}} \exp\left(-\frac{(x+\mu)^T\Sigma^{-1}(x+\mu)}{2}\right)
			\right) \nonumber \\
			& =  \log\left( \frac{1}{\sqrt{|2\pi\Sigma|}} \right) 
			+ \log\left( {(1-q) \times  \exp\left(-\frac{ (x-\mu)^T\Sigma^{-1}(x-\mu) }{2}\right)
				+ q \times  \exp\left(-\frac{ (x+\mu)^T\Sigma^{-1}(x+\mu) }{2}\right)} \right)  \nonumber \\
			&\geq  \log\left( \frac{1}{ \sqrt{|2\pi\Sigma|} } \right) 
			+ \log\left( \exp\left((1-q)\times  \left(-\frac{ (x-\mu)^T\Sigma^{-1}(x-\mu) }{2}\right) + q\times \left(-\frac{ (x+\mu)^T\Sigma^{-1}(x+\mu) }{2}\right)\right) \right)  \nonumber \\ 
			& =  \log\left( \frac{1}{ \sqrt{|2\pi\Sigma|} } \right) 
			+ \left((1-q)\times  \left(-\frac{(x-\mu)^T\Sigma^{-1}(x-\mu)}{2}\right) + q\times \left(-\frac{(x+\mu)^T\Sigma^{-1}(x+\mu)}{2}\right)\right)  \nonumber \\
			& = \log\left( \frac{1}{ \sqrt{|2\pi\Sigma|} } \right) 
			+ \left( -\frac{1}{2}x^T\Sigma^{-1}x - \frac{1}{2}\mu^T\Sigma^{-1}\mu + (1-2q)x^T\Sigma^{-1}\mu \right).
		\end{align}
	\end{minipage}
\end{table*}

Since 
\begin{align*}
	& \int_\mbR p_X(x)
	\left( \log\left( \frac{1}{ \sqrt{|2\pi\Sigma|} } \right) 
	-  \frac{1}{2}\mu^T\Sigma^{-1}\mu \right) dx \\
	& = \log\left( \frac{1}{ \sqrt{|2\pi\Sigma|} } \right) 
	-  \frac{1}{2}\mu^T\Sigma^{-1}\mu,
\end{align*}
and
\begin{align*}
	& \int_{\mbR^n} p_X(x) (1-2q)x^T\Sigma^{-1}\mu dx\\
	& = (1-q)(1-2q)\mu^T\Sigma^{-1}\mu + q (1-2q)(-\mu)^T\Sigma^{-1}\mu \\
	& = (1-2q)^2 \mu^T\Sigma^{-1}\mu,
\end{align*}
and 
\begin{align*}
	\int_{\mbR^n} p_X(x) \left(-\frac{1}{2}x^T\Sigma^{-1}x \right)dx
	= - \frac{1}{2}\left( n + \mu^T\Sigma^{-1}\mu \right),
\end{align*}
where we used Lemma \ref{Lem:QuadraticFormExpectation}, then we have from \eqref{Eq:XEntropy}, 
\begin{align}\label{Eq:XEntropyUb}
	& h(X) \nonumber\\
	& \leq \int_{\mbR^n} - p_X(x) \times \left( \log\left( \frac{1}{ \sqrt{|2\pi\Sigma|} } \right)  -\frac{1}{2}x^T\Sigma^{-1}x  \right) dx \nonumber\\
	& \quad + \int_{\mbR^n} - p_X(x) \times \left(  - \frac{1}{2}\mu^T\Sigma^{-1}\mu + (1-2q)x^T\Sigma^{-1}\mu \right) dx \nonumber\\
	& =(-1) \left( \log\left( \frac{1}{ \sqrt{|2\pi\Sigma|} } \right) 
	-  \frac{1}{2}\mu^T\Sigma^{-1}\mu \right) \nonumber\\
	& \quad + (-1) \left(  (1-2q)^2 \mu^T\Sigma^{-1}\mu - \frac{1}{2}\left( n + \mu^T\Sigma^{-1}\mu \right) \right) \nonumber\\
	& = \frac{1}{2}\log|2\pi e\Sigma| + 4q(1-q)\mu^T\Sigma^{-1} \mu.
\end{align}
Thus, combining \eqref{Defn:MIinBinaryClassificationModel}, \eqref{Eq:ConditionalEntropy_XcY}, and \eqref{Eq:XEntropyUb}, we can get 
\begin{align}
	& I(X;Y) \nonumber \\ 
	& \leq \frac{1}{2}\log|2\pi e\Sigma| + 4q(1-q)\mu^T\Sigma^{-1} \mu -  \frac{1}{2}\log|2\pi e\Sigma| \nonumber \\
	& = 4q(1-q)\mu^T\Sigma^{-1} \mu.
\end{align}

We now derive the lower bound of $I(X;Y)$. Since we have \eqref{Eq:SomeEq},
\begin{table*}
	\centering 
	\begin{minipage}{0.75\textheight}
		\begin{align}\label{Eq:SomeEq}
			& \log\left( 
			\frac{(1-q)}{\sqrt{|2\pi\Sigma|}} \exp\left(-\frac{ (x-\mu)^T\Sigma^{-1}(x-\mu) }{2}\right)
			+  \frac{q}{\sqrt{|2\pi\Sigma|}} \exp\left(-\frac{(x+\mu)^T\Sigma^{-1}(x+\mu)}{2}\right)
			\right) \nonumber \\
			& =  \log\left( \frac{1}{\sqrt{|2\pi\Sigma|}} \right) 
			+ \log\left( {(1-q) \times  \exp\left(-\frac{ (x-\mu)^T\Sigma^{-1}(x-\mu) }{2}\right)
				+ q \times  \exp\left(-\frac{ (x+\mu)^T\Sigma^{-1}(x+\mu) }{2}\right)} \right) \nonumber \\
			& \leq \log\left( \frac{1}{\sqrt{|2\pi\Sigma|}} \right) 
			+ \log\left( \max\left(\exp\left(-\frac{ (x-\mu)^T\Sigma^{-1}(x-\mu) }{2}\right),  \exp\left(-\frac{ (x+\mu)^T\Sigma^{-1}(x+\mu) }{2}\right)\right) \right) \nonumber \\
			& = \log\left( \frac{1}{\sqrt{|2\pi\Sigma|}} \right) 
			+ \log\left( \exp\left(\max\left(-\frac{ (x-\mu)^T\Sigma^{-1}(x-\mu) }{2},-\frac{ (x+\mu)^T\Sigma^{-1}(x+\mu) }{2}\right)\right) \right) \nonumber \\
			& = \log\left( \frac{1}{\sqrt{|2\pi\Sigma|}} \right)  + \max\left(-\frac{ (x-\mu)^T\Sigma^{-1}(x-\mu) }{2},-\frac{ (x+\mu)^T\Sigma^{-1}(x+\mu) }{2}\right)
		\end{align}
	\end{minipage}
\end{table*}
then we get \eqref{Eq:DifferentialEntropy3} where we used \eqref{Eq:SomeEqI_-} and \eqref{Eq:SomeEqI_+}.
\begin{table*}
	\centering 
	\begin{minipage}{0.75\textheight}
		\begin{align}\label{Eq:DifferentialEntropy3}
			h(X) 
			& \geq \int_\mbR - p_X(x)\log\left( \frac{1}{\sqrt{|2\pi\Sigma|}} \right) dx \nonumber \\
			&\quad + \int_\mbR - p_X(x)\max\left(-\frac{ (x-\mu)^T\Sigma^{-1}(x-\mu) }{2},-\frac{ (x+\mu)^T\Sigma^{-1}(x+\mu) }{2}\right) dx \nonumber \\
			& = - \log\left( \frac{1}{\sqrt{|2\pi\Sigma|}} \right)   +  \min\left(I_-,I_+\right),
		\end{align}
		\begin{align}\label{Eq:SomeEqI_-}
			I_-
			&:=\int_{\mbR^n} p_X(x) \frac{ (x-\mu)^T\Sigma^{-1}(x-\mu) }{2} dx \nonumber \\
			&=\int_\mbR 	\left(\frac{(1-q)}{\sqrt{|2\pi\Sigma|}} \exp\left(-\frac{ (x-\mu)^T\Sigma^{-1}(x-\mu) }{2}\right)
			+  \frac{q}{\sqrt{|2\pi\Sigma|}} \exp\left(-\frac{(x+\mu)^T\Sigma^{-1}(x+\mu)}{2}\right)\right) \nonumber\\
			&\quad 	\times \frac{ (x-\mu)^T\Sigma^{-1}(x-\mu) }{2} dx
		\end{align}
		\begin{align}\label{Eq:SomeEqI_+}
			I_+
			&:=
			\int_{\mbR^n} p_X(x) \frac{ (x+\mu)^T\Sigma^{-1}(x+\mu) }{2} dx  \nonumber\\
			& =\int_\mbR 	\left(\frac{(1-q)}{\sqrt{|2\pi\Sigma|}} \exp\left(-\frac{ (x-\mu)^T\Sigma^{-1}(x-\mu) }{2}\right) 
			+  \frac{q}{\sqrt{|2\pi\Sigma|}} \exp\left(-\frac{(x+\mu)^T\Sigma^{-1}(x+\mu)}{2}\right)\right)  \nonumber\\
			&\quad \times \frac{ (x+\mu)^T\Sigma^{-1}(x+\mu) }{2} dx.
		\end{align}
	\end{minipage}
\end{table*}

From Lemma \ref{Lem:QuadraticFormExpectation}, we have 
\begin{align}\label{Eq:I_Minus}
	L_- 
	& = \frac{(1-q)n+q(n+4\mu^T\Sigma^{-1}\mu)}{2} \nonumber\\
	& = \frac{n}{2}+2q\mu^T\Sigma^{-1}\mu,
\end{align}
and 
\begin{align}\label{Eq:I_Plus}
	I_+
	& = \frac{(1-q)(n+4\mu^T\Sigma^{-1}\mu) + qn}{2} \nonumber\\
	& = \frac{n}{2} + 2(1-q)\mu^T\Sigma^{-1}\mu.
\end{align}

Combining  \eqref{Eq:DifferentialEntropy3},  \eqref{Eq:I_Plus}, and \eqref{Eq:I_Minus}, we have 
\begin{align}
	h(X) \geq  - \log\left( \frac{1}{\sqrt{|2\pi\Sigma|}} \right) + \frac{n}{2} + 2\min(q,1-q)\mu^T\Sigma^{-1}\mu.
\end{align}
Then from \eqref{Defn:MIinBinaryClassificationModel}, we have 
\begin{align*}
	I(X;Y) \geq 2\min(q,1-q)\mu^T\Sigma^{-1}\mu.
\end{align*}

\end{proof}

\begin{cor}\label{Corola:ErrorProbabilityBound}
For the data distribution defined in \eqref{Defn:BinaryClassificationDataModel}, we assume the $Y\to X\to \hat{Y}$ forms a Markov chain where $\hat{Y}$ is the prediction from a classifier, and we follow the learning process in Figure \ref{Fig:ErrorProbability} to learn the classifier. Then, the error probability for an arbitrary classifier must satisfy
\begin{align*}
	\max\left( 0,\frac{2+H(Y) - 4q(1-q)\mu^T\Sigma^{-1}\mu - a}{4} \right)
	\leq P_{error},
\end{align*}
where $a:=\sqrt{(H(Y) - I(X;Y) - 2)^2 + 4}$.
\end{cor}

\begin{proof}
(of Corollary \ref{Corola:ErrorProbabilityBound}) 
We can simply plug in the bounds of MI from Theorem \ref{Thm:BinaryClassificationDataModelTruthMI} to Theorem \ref{Thm:EntEstBound} to get Corollary \ref{Corola:ErrorProbabilityBound}.

\end{proof}

\section*{Appendix H: Experimental Results}

\subsection*{Ignored Label Conditional Entropy in ImageNet}\label{Sec:ImgNt_IgnrLblCE}

In this section, we give examples in Figure \ref{Fig:CIFAR10_samples} and \ref{Fig:ImageNetImgs} from CIFAR-10 and ImageNet dataset to show the information loss during the annotation process \cite{deng_imagenet:_2009,he_deep_2015}.

\begin{figure}[!htb]
\centering
\begin{subfigure}[b]{0.11\textwidth}
	\centering
	\includegraphics[width=\linewidth]{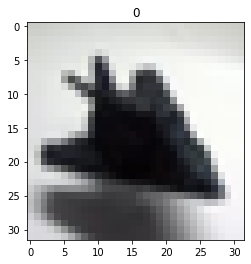}
	\caption{Truth label: class 0 airplane}
\end{subfigure}%
~
\begin{subfigure}[b]{0.11\textwidth}
	\centering
	\includegraphics[width=\linewidth]{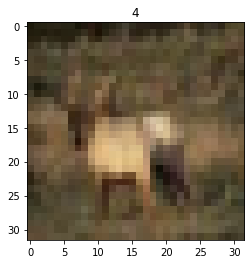}
	\caption{Truth label: class 4 dear}
\end{subfigure}
~
\begin{subfigure}[b]{0.11\textwidth}
	\centering
	\includegraphics[width=\linewidth]{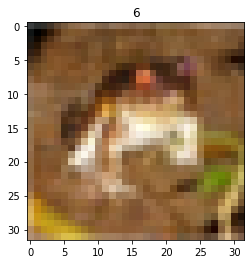}
	\caption{Truth label: class 6 frog}
\end{subfigure}%
~
\begin{subfigure}[b]{0.11\textwidth}
	\centering
	\includegraphics[width=\linewidth]{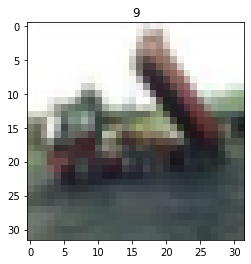}
	\caption{Truth label: class 9 truck}
\end{subfigure}
\caption{Examples from CIFAR-10 dataset.}\label{Fig:CIFAR10_samples}
\end{figure}

\begin{figure}[!htb]
\centering
\begin{subfigure}[b]{0.22\textwidth}
	\centering
	\includegraphics[width=\linewidth]{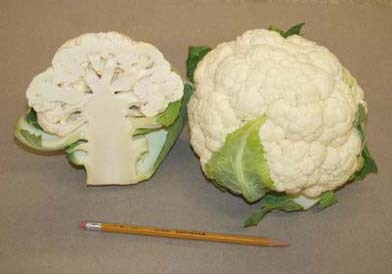}
	\caption{Truth label is cauliflower}\label{Fig:ImageNetImgsMultiple}
\end{subfigure}%
~
\begin{subfigure}[b]{0.22\textwidth}
	\centering
	\includegraphics[width=\linewidth]{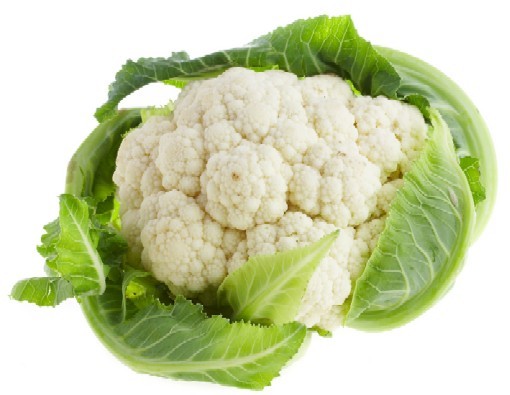}
	\caption{Truth label is cauliflower}\label{Fig:ImageNetImgsSingle}
\end{subfigure}
\caption{Image examples from ImageNet dataset. Though Figure \ref{Fig:ImageNetImgsMultiple} contains also a pencil, the human annotator only labeled it with cauliflower. In deep learning practice, since we usually use the one-hot encoding of the label for a given image, i.e., assigning all the probability mass to the annotated class while zero to all the other classes, this further encourages the model to ignore the conditional entropy of the label \cite{he_deep_2015,szegedy_rethinking_2016,yi_trust_2019}.}
\label{Fig:ImageNetImgs}
\end{figure}

\subsection*{Error Probability Lower Bound via Mutual Information in Imbalanced Dataset}\label{Sec:ErrPrbBd_ImblcdDs}

In Figure \ref{Fig:ErrroProbability_and_MI}, we show the relation between the classification error probability lower bound in terms of mutual information for a imbalanced data distribution.

\begin{figure}[!htb]
\centering
\includegraphics[width=\linewidth]{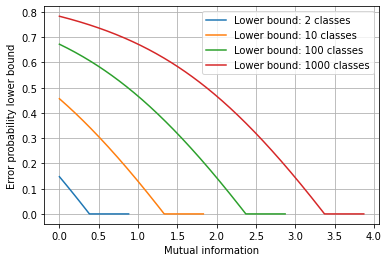}
\caption{Error probability lower bound and mutual information for unbalanced data disrtribution: one class takes probability mass 0.7, while the other classes share the 0.3 evenly.}\label{Fig:ErrroProbability_and_MI}
\end{figure}

\subsection*{Illustration of Mutual Information Bounds in Binary Classification Data Model}\label{Sec:BnrClsDtMdl}

We consider a simplified case in $\mbR$ with $\mu=1$ and variance $\sigma^2=1$ for illustrating the mutual information bounds, when the variance becomes bigger, the two distributions $\mcN(1,\sigma^2)$ and $\mcN(-1,\sigma^2)$ get closer to each other. Thus, conditioning on $X$ can give very little information about $Y$, making it difficult to differentiate the two class labels. In Figure \ref{Fig:XConditionalDistributionsUnderDifferentVariance}, we give illustrations for this phenomenon. As we can see in Figure \ref{Fig:XConditionalDistributionsUnderDifferentVariance}, as the $\sigma^2$ increases from 1 to 100, the two conditional distributions $\mcN(1,\sigma^2)$ and $\mcN(-1,\sigma^2)$ get closer to each other, and the marginal distribution in \eqref{Eq:XMarginalDistribution} can be approximated by either of them. This results in that less information is revealed about $Y$ when we condition on $X$ with larger variance $\sigma^2$. We also give illustrations of the mutual information bounds for the data distribution in $\mbR$ in Figure \ref{Fig:MutualInformationBounds}. 

\begin{figure*}[!htb]
\centering
\begin{subfigure}[b]{0.42\textwidth}
	\centering
	\includegraphics[width=\linewidth]{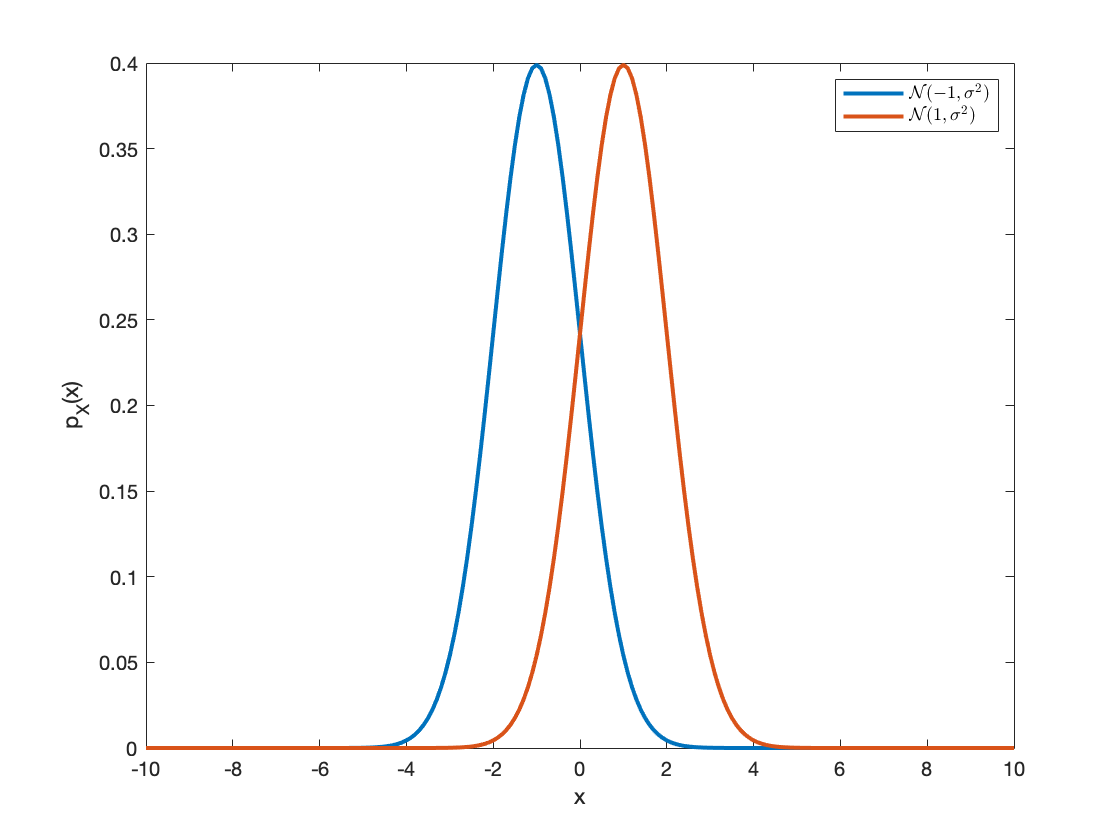}
	\caption{Variance $\sigma^2=1$}
\end{subfigure}%
~
\begin{subfigure}[b]{0.42\textwidth}
	\centering
	\includegraphics[width=\linewidth]{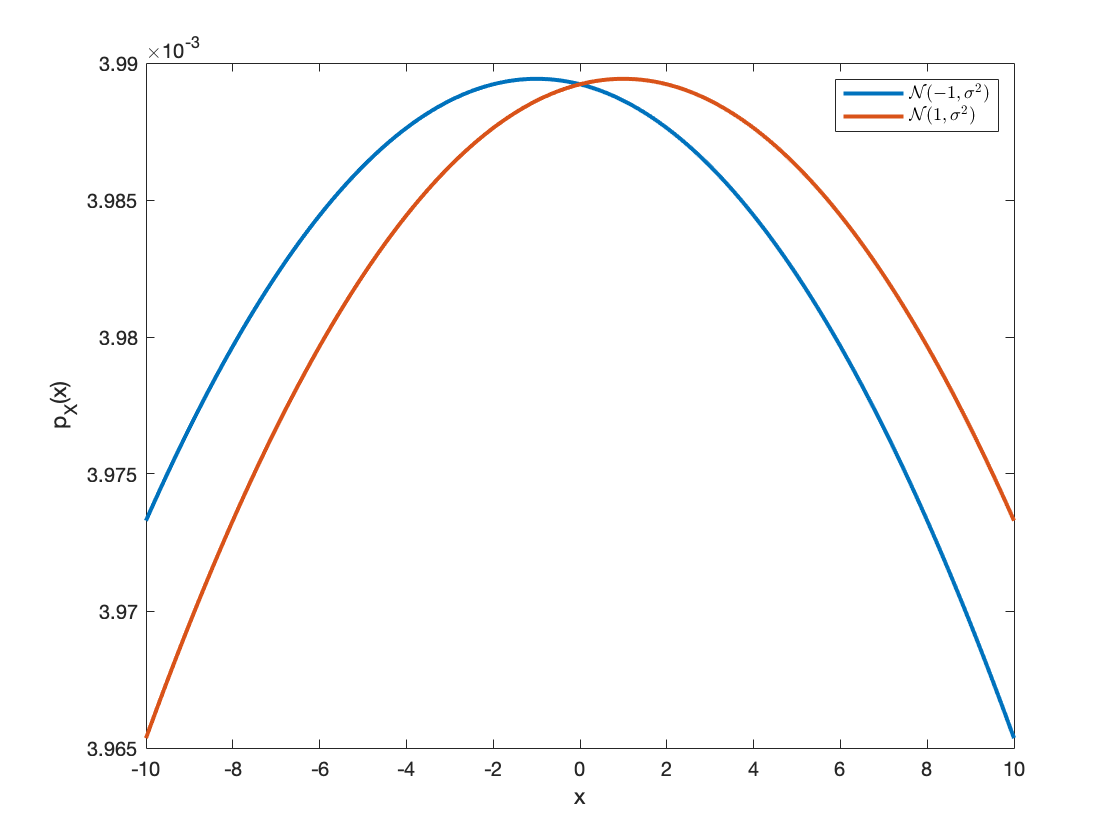}
	\caption{Variance $\sigma^2=100$}
\end{subfigure}
\caption{Conditional distribution $p_{X|Y}$ under with different variance $\sigma^2$.}\label{Fig:XConditionalDistributionsUnderDifferentVariance}
\end{figure*}

\begin{figure*}[!htb]
\centering
\begin{subfigure}[b]{0.42\textwidth}
	\centering
	\includegraphics[width=\linewidth]{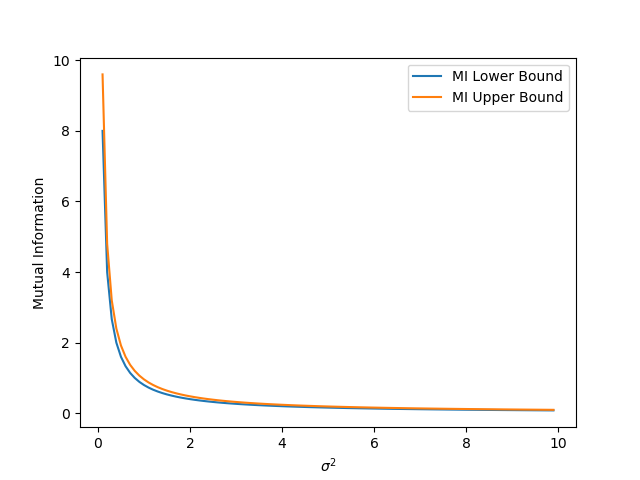}
	\caption{$q=0.4$}
\end{subfigure}%
~
\begin{subfigure}[b]{0.42\textwidth}
	\centering
	\includegraphics[width=\linewidth]{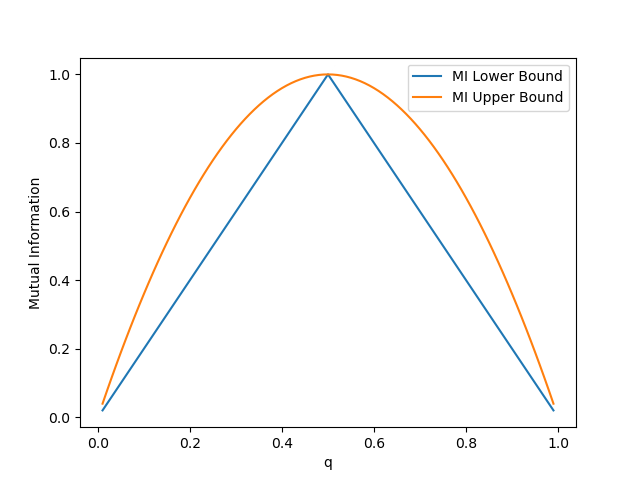}
	\caption{$\sigma^2=1$}
\end{subfigure}
\caption{Mutual information bounds of binary classification data model in \eqref{Defn:BinaryClassificationDataModel}.}\label{Fig:MutualInformationBounds}
\end{figure*}

\subsection*{Real-word Datasets and Corresponding Model Architectures}

{\bf MLP and CNN architectures for MNIST classification} The MLP is a 3-layer fully connected neural network with 64, 64, and 10 neurons in each layer. All the layers except the final layer use a relu activation function. The CNN is a 4-layer neural network with 2 convolutional layers followed by 2 fully connected layers. The first convolutional layer has 10 kernels of size $5\times5$, and the second convolutional layer has 20 kernels of size $5\times5$. A maxpooling layer with stride 2 is applied after each convolutional layer. The two fully connected layers have 50 and 10 neurons, respectively, and the first fully connected layer uses relu activation function. All our experiments are conducted on a Windows machine with Intel Core(TM) i9 CPU @ 3.7GHz, 64Gb RAM, and 1 NVIDIA RTX 3090 GPU card. 

\subsection*{Learning Curves on MNIST and CIFAR-10}

In this section, we present supplemental experimental results in Figure \ref{Fig:MNIST_MLP} and \ref{Fig:MNIST_CNN} for demonstrating the performance of the proposed approach in MNIST classification, and in Figure \ref{Fig:CIFAR10_ResNet18} and \ref{Fig:CIFAR10_GoogLeNet} for demonstrating the performance of the proposed approach in CIFAR-10 classification with ResNet-18 and GoogLeNet \cite{szegedy_going_2015,he_deep_2015}.

\begin{figure*}[!htb]
\centering
\begin{subfigure}[b]{0.42\textwidth}
	\centering
	\includegraphics[width=\linewidth]{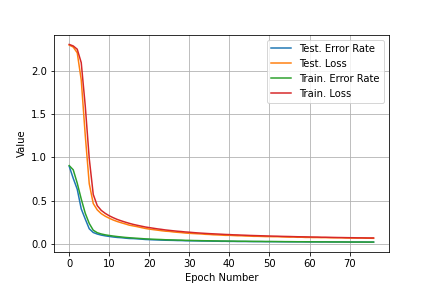}
	\caption{celLoss}\label{Fig:MNIST_CNN_cel}
\end{subfigure}%
~
\begin{subfigure}[b]{0.42\textwidth}
	\centering
	\includegraphics[width=\linewidth]{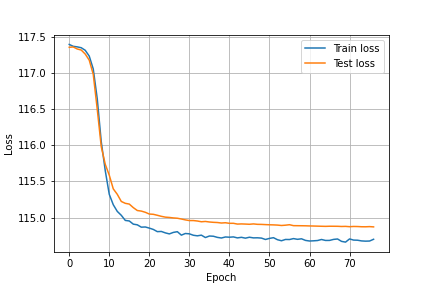}
	\caption{milLoss}\label{Fig:MNIST_CNN_milLoss}
\end{subfigure}
~
\begin{subfigure}[b]{0.42\textwidth}
	\centering
	\includegraphics[width=\linewidth]{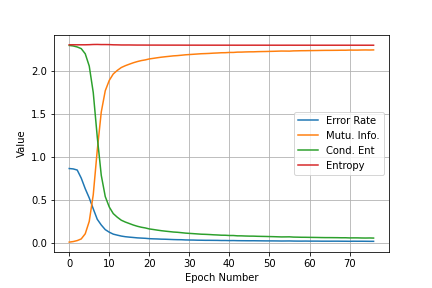}
	\caption{milLoss}\label{Fig:MNIST_CNN_milInformation}
\end{subfigure}%
\\
\begin{subfigure}[b]{0.42\textwidth}
	\centering
	\includegraphics[width=\linewidth]{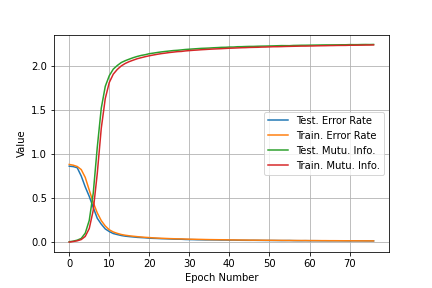}
	\caption{milLoss}\label{Fig:MNIST_CNN_milErrorRate_MI}
\end{subfigure}
\begin{subfigure}[b]{0.42\textwidth}
	\centering
	\includegraphics[width=\linewidth]{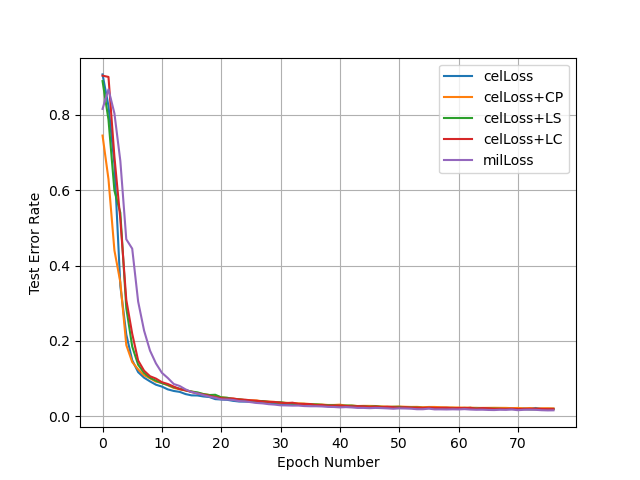}
	\caption{Error rates}\label{Fig:MNIST_CNN_milAcel_ErrorRate}
\end{subfigure}
\caption{CNN on MNIST. \ref{Fig:MNIST_CNN_cel}: error rate and loss during training and test at different epochs. \ref{Fig:MNIST_CNN_milLoss}: loss during training and testing at different epochs. \ref{Fig:MNIST_CNN_milInformation}: error rate, mutual information, label conditional entropy, and label entropy during test at different epochs. \ref{Fig:MNIST_CNN_milErrorRate_MI}: mutual information and error rate during training and testing at different epoch. \ref{Fig:MNIST_CNN_milAcel_ErrorRate}: testing error rate curve associated with different loss function.}\label{Fig:MNIST_CNN}
\end{figure*}

\begin{figure*}[!htb]
\centering
\begin{subfigure}[b]{0.42\textwidth}
	\centering
	\includegraphics[width=\linewidth]{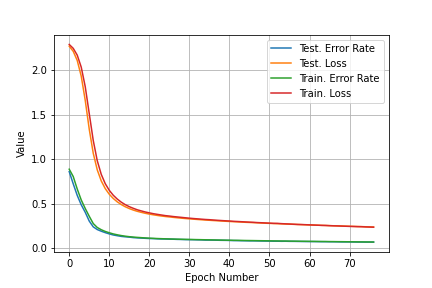}
	\caption{celLoss}\label{Fig:MNIST_MLP_cel}
\end{subfigure}%
~
\begin{subfigure}[b]{0.42\textwidth}
	\centering
	\includegraphics[width=\linewidth]{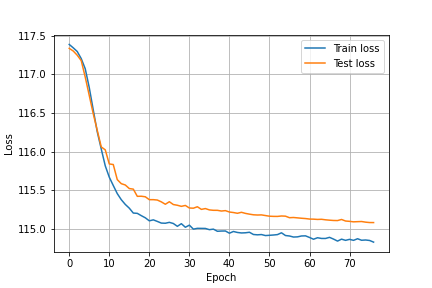}
	\caption{milLoss}\label{Fig:MNIST_MLP_milLoss}
\end{subfigure}
\begin{subfigure}[b]{0.42\textwidth}
	\centering
	\includegraphics[width=\linewidth]{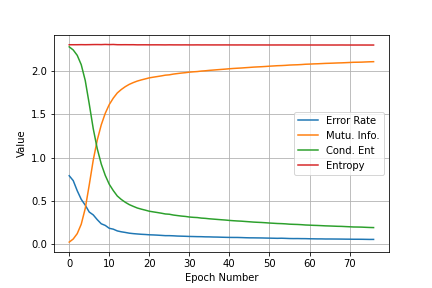}
	\caption{milLoss}\label{Fig:MNIST_MLP_milInformation}
\end{subfigure}%
~
\begin{subfigure}[b]{0.42\textwidth}
	\centering
	\includegraphics[width=\linewidth]{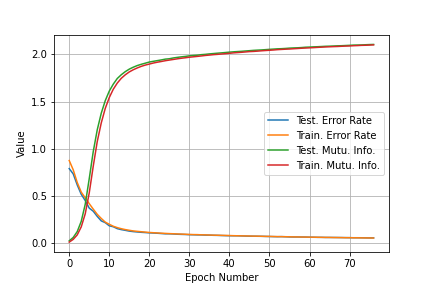}
	\caption{milLoss}\label{Fig:MNIST_MLP_milTrainTest}
\end{subfigure}
~
\begin{subfigure}[b]{0.42\textwidth}
	\centering
	\includegraphics[width=\linewidth]{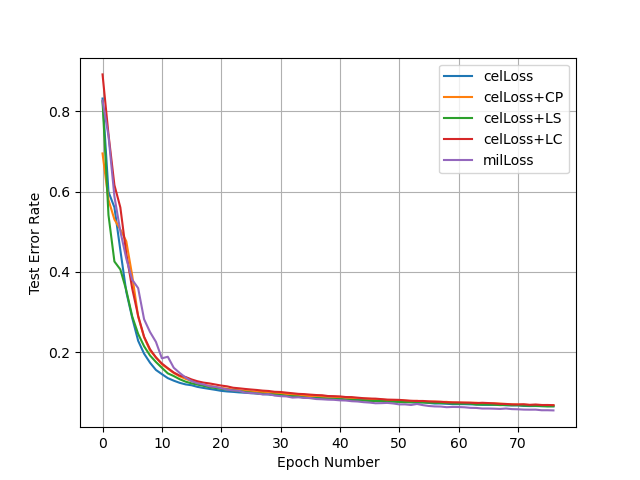}
	\caption{Error rates}\label{Fig:MNIST_MLP_testErrorRate}
\end{subfigure}
\caption{Multiple layer perceptron on MNIST. \ref{Fig:MNIST_MLP_cel}: error rate and loss during training and test at different epochs. \ref{Fig:MNIST_MLP_milLoss}: loss during training and testing at different epochs. \ref{Fig:MNIST_MLP_milInformation}: error rate, mutual information, label conditional entropy, and label entropy during test at different epochs. \ref{Fig:MNIST_MLP_milTrainTest}: mutual information and error rate during training and testing at different epoch. \ref{Fig:MNIST_MLP_testErrorRate}: testing error rate curves associated with different loss functions.}\label{Fig:MNIST_MLP}
\end{figure*}

\begin{figure*}[!htb]
\centering
\begin{subfigure}[b]{0.42\textwidth}
	\centering
	\includegraphics[width=\linewidth]{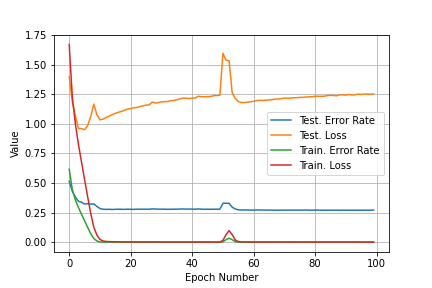}
	\caption{celLoss}\label{Fig:CIFAR10_ResNet18_cel}
\end{subfigure}%
~
\begin{subfigure}[b]{0.42\textwidth}
	\centering
	\includegraphics[width=\linewidth]{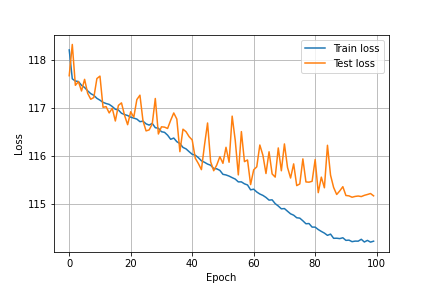}
	\caption{milLoss}\label{Fig:CIFAR10_ResNet18_milLoss}
\end{subfigure}
\begin{subfigure}[b]{0.42\textwidth}
	\centering
	\includegraphics[width=\linewidth]{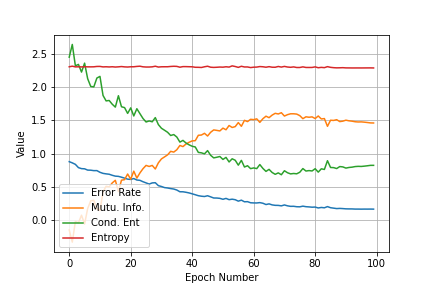}
	\caption{milLoss}\label{Fig:CIFAR10_ResNet18_milInformation}
\end{subfigure}%
~
\begin{subfigure}[b]{0.42\textwidth}
	\centering
	\includegraphics[width=\linewidth]{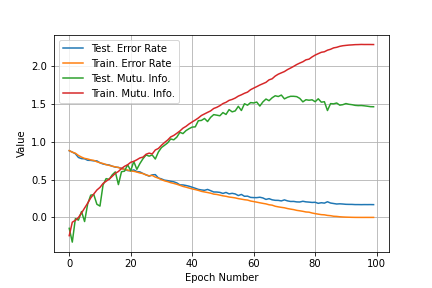}
	\caption{milLoss}\label{Fig:CIFAR10_ResNet18_milErrorRate_MI}
\end{subfigure}
\begin{subfigure}[b]{0.42\textwidth}
	\centering
	\includegraphics[width=\linewidth]{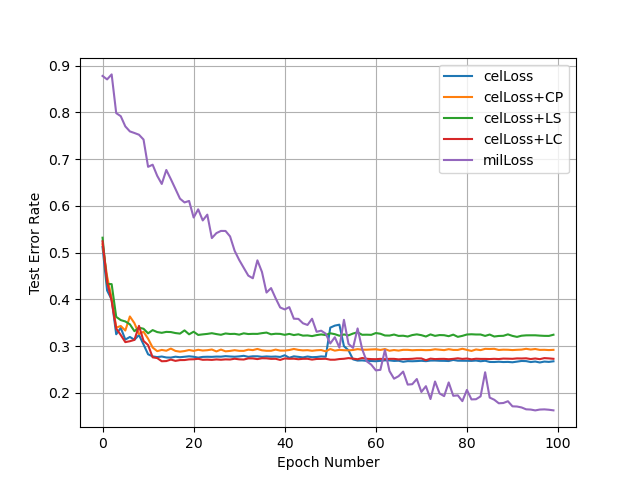}
	\caption{Error rates}\label{Fig:CIFAR10_ResNet18_milAcel_ErrorRate}
\end{subfigure}
\caption{ResNet18 on CIFAR-10. \ref{Fig:CIFAR10_ResNet18_cel}: error rate and loss during training and test at different epochs. \ref{Fig:CIFAR10_ResNet18_milLoss}: loss during training and testing at different epochs. \ref{Fig:CIFAR10_ResNet18_milInformation}: error rate, mutual information, label conditional entropy, and label entropy during test at different epochs. \ref{Fig:CIFAR10_ResNet18_milErrorRate_MI}: mutual information and error rate during training and testing at different epoch. \ref{Fig:CIFAR10_ResNet18_milAcel_ErrorRate}: testing error rate curves associated with different loss functions.}\label{Fig:CIFAR10_ResNet18}
\end{figure*}

\begin{figure*}[!htb]
\centering
\begin{subfigure}[b]{0.42\textwidth}
	\centering
	\includegraphics[width=\linewidth]{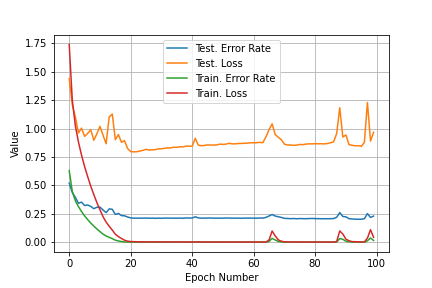}
	\caption{celLoss}\label{Fig:CIFAR10_GoogLeNet_cel}
\end{subfigure}%
~
\begin{subfigure}[b]{0.42\textwidth}
	\centering
	\includegraphics[width=\linewidth]{Figures/GoogLeNet_mil_avg/GoogLeNetmilLoss_epoch_loss}
	\caption{milLoss}\label{Fig:CIFAR10_GoogLeNet_milLoss}
\end{subfigure}
\begin{subfigure}[b]{0.42\textwidth}
	\centering
	\includegraphics[width=\linewidth]{Figures/GoogLeNet_mil_avg/GoogLeNetmilLoss_Testing_MiAccLossEntCondent_vs_Epoch_noLoss}
	\caption{milLoss}\label{Fig:CIFAR10_GoogLeNet_milInformation}
\end{subfigure}%
~
\begin{subfigure}[b]{0.42\textwidth}
	\centering
	\includegraphics[width=\linewidth]{Figures/GoogLeNet_mil_avg/GoogLeNetmilLoss_TrainTesting_MiAccloss_vs_Epoch_noLoss}
	\caption{milLoss}\label{Fig:CIFAR10_GoogLeNet_milErrorRate_MI}
\end{subfigure}
\begin{subfigure}[b]{0.42\textwidth}
	\centering
	\includegraphics[width=\linewidth]{Figures/CIFAR10_GoogLeNet_testErrorRate_acrossLossFunctions}
	\caption{Error rates}\label{Fig:CIFAR10_GoogLeNet_milAcel_ErrorRate}
\end{subfigure}
\caption{GoogLeNet on CIFAR-10. \ref{Fig:CIFAR10_GoogLeNet_cel}: error rate and loss during training and test at different epochs. \ref{Fig:CIFAR10_GoogLeNet_milLoss}: loss during training and testing at different epochs. \ref{Fig:CIFAR10_GoogLeNet_milInformation}: error rate, mutual information, label conditional entropy, and label entropy during test at different epochs. \ref{Fig:CIFAR10_GoogLeNet_milErrorRate_MI}: mutual information and error rate during training and testing at different epoch. \ref{Fig:CIFAR10_GoogLeNet_milAcel_ErrorRate}: testing error rate curves associated with different loss functions.}\label{Fig:CIFAR10_GoogLeNet}
\end{figure*}

\end{document}